\documentclass{article}

\usepackage[final, nonatbib]{neurips_2020}

\usepackage[utf8]{inputenc}
\usepackage[T1]{fontenc}
\usepackage{hyperref}
\usepackage{url}
\usepackage{booktabs}
\usepackage{nicefrac}
\usepackage{microtype}
\usepackage{graphicx}
\usepackage{subfigure}

\usepackage{amsmath,amsfonts,amssymb,amsthm,mathabx}
\usepackage{multirow}
\usepackage{dsfont}
\usepackage{comment}
\usepackage[export]{adjustbox}
\usepackage{xcolor}
\usepackage{rotating}
\usepackage{todonotes}
\usepackage{adjustbox}
\usepackage{enumitem}
\usepackage{tcolorbox}

\newcommand{\J}{{\mathbf J}}
\newcommand{\Jac}[1]{\J_f({#1})}

\newcommand{\E}{{\mathbf E}}

\renewcommand{\Re}{{\mathbb R}}	
	
\newcommand{\mX}{{\mathbf X}}	
\newcommand{\w}{{\mathbf w}}	
\renewcommand{\v}{{\mathbf v}}	
\newcommand{\x}{{\mathbf x}}	
\newcommand{\z}{{\mathbf z}}	
\newcommand{\y}{{\mathbf y}}	
\renewcommand{\u}{{\mathbf u}}
\renewcommand{\v}{{\mathbf v}}
\renewcommand{\b}{{\mathbf b}}
\newcommand{\W}{{\mathbf W}}
	
\newtheorem{lemma}{Lemma}

\DeclareMathOperator*{\argmin}{arg\,min}
\DeclareMathOperator*{\argmax}{arg\,max}

\newcommand{\uSet}{{\mathcal U}}

\newtheorem{theorem}{Theorem}

\graphicspath{{./figures/}}

\title{Adversarial Training is a Form of Data-dependent Operator Norm Regularization}

\author{
  Kevin Roth$^*$\\ 
  Dept of Computer Science \\
  ETH Z\"urich\\
  \texttt{\small kevin.roth@inf.ethz.ch} \\
  \And
  Yannic Kilcher$^*$\\
  Dept of Computer Science \\
  ETH Z\"urich\\
  \texttt{\small yannic.kilcher@inf.ethz.ch} \\
  \And
  Thomas Hofmann \\
  Dept of Computer Science \\
  ETH Z\"urich\\
  \texttt{\small thomas.hofmann@inf.ethz.ch} \\   
}

\begin{document}

\maketitle

\vspace{-2mm}
\begin{abstract}
\vspace{-1mm}
We establish a theoretical link between adversarial training and operator norm regularization for deep neural networks. 
Specifically, we prove that $\ell_p$-norm constrained projected gradient ascent based adversarial training with an $\ell_q$-norm loss on the logits of clean and perturbed inputs is equivalent to data-dependent (p, q) operator norm regularization.
This fundamental connection confirms the long-standing argument that a network's sensitivity to adversarial examples is tied to its spectral properties and hints at novel ways to robustify and defend against adversarial attacks.
We provide extensive empirical evidence on state-of-the-art network architectures to support our theoretical results. 
\end{abstract}


\vspace{-2mm}
\section{Introduction}
While deep neural networks are known to be robust to random noise, it has been shown that their accuracy dramatically deteriorates in the face of so-called adversarial examples \cite{biggio2013evasion, szegedy2013intriguing, goodfellow2014explaining}, i.e.\ small perturbations of the input signal, often imperceptible to humans, that are sufficient to induce large changes in the model output. 
This apparent vulnerability is worrisome as deep nets start to proliferate in the real-world, including in safety-critical deployments.

The most direct strategy of robustification, called adversarial training, aims to robustify a machine learning model by training it against an adversary that perturbs the examples before passing them to the model \cite{goodfellow2014explaining, kurakin2016adversarial, miyato2015distributional, miyato2018virtual, madry2018towards}.
A different strategy of defense is to detect whether the input has been perturbed, by detecting characteristic regularities either in the adversarial perturbations themselves or in the network activations they induce \cite{grosse2017statistical, feinman2017detecting, xu2017feature, metzen2017detecting, carlini2017adversarial, roth2019odds}.

Despite practical advances in finding adversarial examples and defending against them, 
it is still an open question whether
(i) adversarial examples are unavoidable, i.e.\ no robust model exists, cf.\ \cite{fawzi2018adversarial, gilmer2018adversarial}, 
(ii) learning a robust model requires too much training data, 
cf.\ \cite{schmidt2018adversarially}, 
(iii) learning a robust model from limited training data is possible but computationally intractable \cite{bubeck2019adversarial}, or 
(iv) we just have not found the right model / training algorithm yet.

\begin{figure}[h!]
\centering
\begin{tabular}{ccc}
\adjustbox{valign=c, raise=-7pt}{$\tilde{\v}_{k} \leftarrow \Jac{\x}^\top \u_{k}$}\hspace{-2mm} &
\adjustbox{valign=c}{\adjincludegraphics[width=0.28\linewidth, trim={{0.27\width} {0\height} {0.3\width} {0.0\height}},clip]{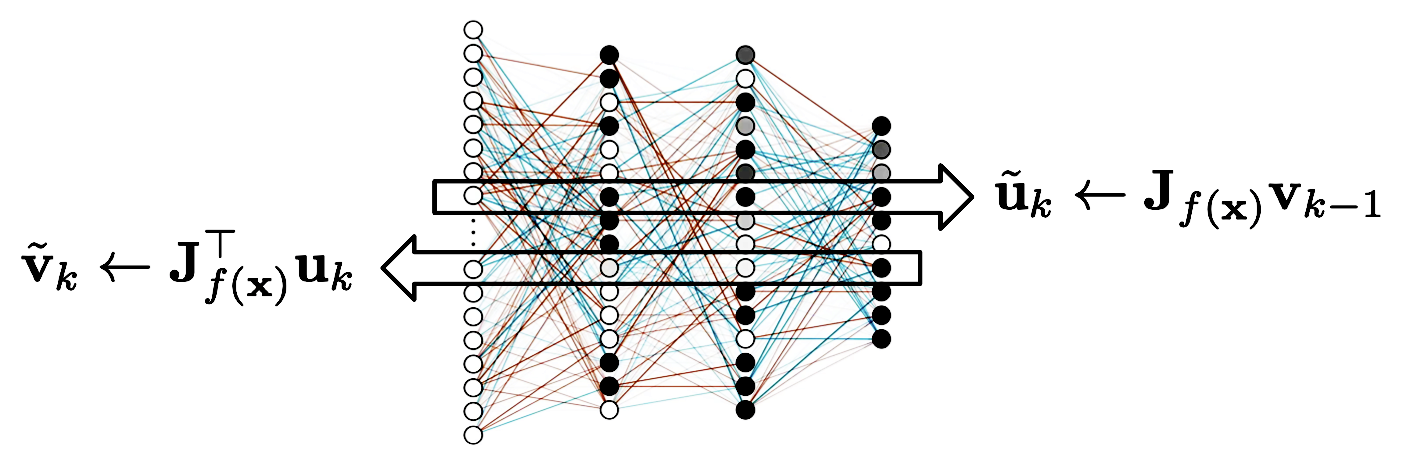}} &
\hspace{-2.5mm}\adjustbox{valign=c, raise=6pt}{$\tilde{\u}_{k} \leftarrow \Jac{\x} \v_{k-1}$} 
\end{tabular}
\caption{
Our theoretical results suggest to think of iterative adversarial attacks as power-method-like forward-backward passes (indicated by the arrows) through (the Jacobian $\Jac{\x}$ of) the network. 
} 
\label{fig:ATExplainer}
\end{figure}

\vspace{-2mm}
In this work, we investigate adversarial vulnerability in neural networks by focusing on the attack algorithms used to find adversarial examples.
In particular, we make the following contributions: 
\begin{itemize}[leftmargin=8mm]
\vspace{-1mm}
\item We present a data-dependent variant of spectral norm regularization that directly regularizes large singular values of a neural network in regions that are supported by the data, as opposed to existing methods that regularize a global, data-independent upper bound. 
\item We prove that $\ell_p$-norm constrained projected gradient ascent based adversarial training with an $\ell_q$-norm loss on the logits of clean and perturbed inputs is equivalent to data-dependent (p, q) operator norm regularization.
\item We conduct extensive empirical evaluations showing among other things that 
(i) adversarial perturbations align with dominant singular vectors,
(ii) adversarial training 
dampens the singular values, and 
(iii) adversarial training and data-dependent spectral norm regularization give rise to models that are significantly more linear around data than normally trained ones. 
\end{itemize}


\enlargethispage{1\baselineskip}
\vspace{-1mm}
\section{Related Work}
\label{sec:relwork}
\vspace{-1mm}
The idea that a conservative measure of the sensitivity of a network to adversarial examples can be obtained by computing the spectral norm of the individual weight layers appeared already in the seminal work of Szegedy et al.\ \cite{szegedy2013intriguing}.
A number of works have since suggested to regularize the global spectral norm \cite{yoshida2017spectral, miyato2018spectral, bartlett2017spectrally, farnia2018generalizable}  
and Lipschitz constant \cite{cisse2017parseval, hein2017formal, tsuzuku2018lipschitz, raghunathan2018certified} 
as a means to improve model robustness.
Input gradient regularization has also been suggested \cite{gu2014towards, lyu2015unified, cisse2017parseval}.
Adversarial robustness has also been investigated via robustness to random noise \cite{fawzi2018analysis, fawzi2016robustness} and decision boundary tilting \cite{tanay2016boundary}.

The most direct and popular strategy of robustification, however, is to use adversarial examples as data augmentation during training
\cite{goodfellow2014explaining, shaham2018understanding, sabour2015adversarial, kurakin2016adversarial, papernot2016transferability, moosavi2016deepfool, miyato2018virtual, madry2018towards}. 
Adversarial training can be viewed as a variant of (distributionally) robust optimization \cite{el1997robust, xu2009robustness, bertsimas2018characterization, namkoong2017variance, sinha2018certifying, gao2016distributionally}
where a machine learning model is trained to minimize the worst-case loss against an adversary that can shift the entire training data within an uncertainty set.
Interestingly, for certain problems and uncertainty sets, such as for linear regression and induced matrix norm balls, robust optimization has been shown to be equivalent to regularization \cite{el1997robust, xu2009robustness, bertsimas2018characterization, bietti2019kernel}. 
Similar results have been obtained also for (kernelized) SVMs~\cite{xu2009robustness}. 

We extend these lines of work by establishing a theoretical link between adversarial training and data-dependent operator norm regularization.
This fundamental connection confirms the long-standing argument that a network's sensitivity to adversarial examples is tied to its spectral properties 
and opens the door for robust generalization bounds via data-dependent operator norm based ones.


\vspace{-1mm}
\section{Background}
\label{sec:background}
\vspace{-1mm}
\textbf{Notation.} In this section we rederive global spectral norm regularization {\`a} la Yoshida \& Miyato \cite{yoshida2017spectral}, while also setting up the notation for later.
Let $\x$ and $y$ denote input-label pairs generated from a data distribution~$P$.
Let $f : \mathcal{X} \subset \Re^n \to \Re^d$ denote the logits of a $\theta$-parameterized piecewise linear classifier, i.e.\
$f(\cdot) = \W^L \phi^{L-1}(\W^{L-1}\phi^{L-2}(\dots) + \b^{L-1}) + \b^L$, where $\phi^\ell$ is the activation function, 
and $\W^\ell$, $\b^\ell$ denote the layer-wise weight matrix\footnote{Convolutional layers can be constructed as matrix multiplications by converting them into a Toeplitz matrix.} and bias vector, collectively denoted~by~$\theta$. 
Let us furthermore assume that each activation function is a ReLU (the argument can easily be generalized to other piecewise linear activations). 
In this case, the activations $\phi^\ell$ act as input-dependent diagonal matrices $\Phi_{\x}^\ell := \text{diag}(\phi^\ell_\x)$, 
where an element in the diagonal $\phi^\ell_\x := \mathbf{1}(\tilde \x^\ell \ge 0)$ is one if the corresponding pre-activation $\tilde \x^\ell := \W^\ell \phi^{\ell-1}(\cdot) + \b^\ell$ is positive and zero otherwise.

Following Raghu et al.\ \cite{raghu2017expressive}, we call $\phi_{\x} := (\phi_{\x}^1, \dots, \phi_{\x}^{L-1}) \in \{0,1\}^m$ the ``activation pattern'', where $m$ is the number of neurons in the network. 
For any activation pattern $\phi \in \{0,1\}^m$ we can define the preimage $X(\phi) := \{ \x \in \Re^n : \phi_{\x} = \phi \}$,
inducing a partitioning of the input space via $\Re^n = \bigcup_\phi X(\phi)$.
Note that some $X(\phi) = \emptyset$, as not all combinations of activiations may be feasible.
See Figure~1 in \cite{raghu2017expressive} or Figure~3 in \cite{novak2018sensitivity} for an illustration of ReLU tesselations of the input space.

\textbf{Linearization.} We can linearize $f$ within a neighborhood around $\x$ as follows 
\begin{align}
f(\x + \Delta \x) \simeq f(\x) + \Jac{\x} \Delta \x\ , \quad (\text{with equality if }\ \x + \Delta \x \in X(\phi_{\x})) ,
\label{eq:linearization}
\end{align}
where $\Jac{\x}$ denotes the Jacobian of $f$ at $\x$
\begin{align}
\Jac{\x} = \W^L \cdot \Phi_{\x}^{L-1} \cdot \W^{L-1} \cdot \Phi_{\x}^{L-2} \cdot\cdot\cdot \Phi_{\x}^{1} \cdot \W^1 \,.
\label{eq:jacobiandecomposition}
\end{align}

For small $|| \Delta \x ||_2 \neq 0$, we have the following bound
\begin{align}
\hspace{-2mm}\frac{|| f(\x \!+\! \Delta \x) \!-\! f(\x) ||_2}{|| \Delta \x ||_2} \simeq 
\frac{|| \Jac{\x} \Delta \x ||_2}{ || \Delta \x ||_2 } \leq \sigma( \Jac{\x} ) := \max_{ \v : || \v ||_2 =1} || \Jac{\x} \v ||_2
\label{eq:spectralnormdefinition}
\end{align}
where $\sigma( \Jac{\x} )$ is the \textit{spectral norm} (largest singular value) of the linear operator $\Jac{\x}$.
From a robustness perspective we want $\sigma( \Jac{\x} )$ to be small in regions that are supported by the data.

\textbf{Global Spectral Norm Regularization.} Based on the factorization in Eq.~\ref{eq:jacobiandecomposition} and the non-expansiveness of the activations, $\sigma(\Phi_{\x}^\ell) \leq 1$, $\forall \ell \in \{1,...,L\!-\!1\}$, 
Yoshida \& Miyato~\cite{yoshida2017spectral} suggested to upper-bound the spectral norm of the Jacobian by the \textit{data-independent}(!) product of the spectral norms of the weight matrices 
$\sigma( \Jac{\x} )  \leq {\prod_{\ell = 1}^L} \sigma(\W^\ell), \ \forall \x \in \mathcal{X}$.
The layer-wise spectral norms $\sigma^\ell := \sigma(\W^\ell)$ 
can be computed iteratively using the power method$^2$. 
Starting from a random $\v_0$, 
\begin{equation}
\begin{aligned}
\u^\ell_k \leftarrow \tilde \u^\ell_k / ||\tilde \u^\ell_k||_2 \ , \ \tilde \u^\ell_k \leftarrow \W^\ell \v^\ell_{k-1} \ , \quad
\v^\ell_k \leftarrow  \tilde \v^\ell_k / ||\tilde \v^\ell_k||_2 \ , \ \tilde \v^\ell_k \leftarrow (\W^\ell)^\top \u^\ell_{k} \,.
\label{eq:dataindependentSN}
\end{aligned}
\end{equation}
The (final) singular value can be computed as $\sigma^\ell_k = (\u^{\ell}_k)^\top \W^\ell \v^\ell_{k}$.

Yoshida \& Miyato \cite{yoshida2017spectral} suggest to turn this upper-bound into a global (data-independent) regularizer
\vspace{-1mm}
\begin{align}
\min \theta \rightarrow \E_{(\x,y )\sim \hat{P}} \left[ \ell(y, f(\x)) \right] + \frac{\lambda}{2} {\sum_{\ell = 1}^L} \sigma(\W^\ell)^2 \ , 
\label{eq:YoshidaReg}
\end{align}
where $\ell(\cdot, \cdot)$ denotes a classification loss.
It can be verified that $\nabla_\W \sigma(\W)^2/2 = \sigma \u \v^\top$, with $\sigma$, $\u$, $\v$ being the principal singular value/vectors of $\W$. Eq.~\ref{eq:YoshidaReg} thus effectively adds a term $\lambda \sigma^\ell \u^\ell (\v^\ell)^\top$ to the parameter gradient of each layer $\ell$. In terms of computational complexity, because the global regularizer decouples from the empirical loss, the power-method can be amortized across data-points, hence a single power method iteration per parameter update step usually suffices in practice \cite{yoshida2017spectral}.

\vspace{-2mm}
\paragraph{Global vs.\ Local Regularization}
\label{sec:globalvslocalregularization}
Global bounds trivially generalize from the training to the test set.
The problem however is that they can be arbitrarily loose, e.g.\ penalizing the spectral norm over irrelevant regions of the ambient space. 
To illustrate this, consider the ideal robust classifier that is essentially piecewise constant on class-conditional regions, with sharp transitions between the classes. 
The global spectral norm will be heavily influenced by the sharp transition zones, 
whereas a local data-dependent bound can adapt to regions where the classifier is approximately constant \cite{hein2017formal}.
We would therefore expect a global regularizer to have the largest effect in the empty parts of the input space.
A local regularizer, on the contrary, has its main effect around the data manifold.

\enlargethispage{1\baselineskip}


\section{Adversarial Training is a Form of Operator Norm Regularization}
\label{sec:ATisaformofONR}
\vspace{-1mm}
\subsection{Data-dependent Spectral Norm Regularization}
\label{sec:datadepSNR}
We now show how to directly regularize the \textit{data-dependent} spectral norm of the Jacobian $\Jac{\x}$.
Assuming that the dominant singular value is non-degenerate\footnote{Due to numerical errors, we can safely assume that the dominant singular value is non-degenerate.}, 
the largest singular value and the corresponding left and right singular vectors can efficiently be computed via the power method. Starting from $\v_0$,
we successively compute the unnormalized (denoted with a tilde) and normalized approximations to the dominant singular vectors,
\begin{equation}
\begin{aligned}
\label{eq:datadepspectralnorm}
\u_{k} &\leftarrow \tilde{\u}_{k} / || \tilde{\u}_{k} ||_2 \ , \ \tilde{\u}_{k} \leftarrow \Jac{\x} \v_{k-1} \\ 
\v_{k} &\leftarrow \tilde{\v}_{k} / || \tilde{\v}_{k} ||_2 \ , \,\ \tilde{\v}_{k} \leftarrow \Jac{\x}^\top \u_{k} \, .
\end{aligned}
\end{equation}
The (final) singular value can then be computed via $\sigma(\Jac{\x}) = \u^\top \Jac{\x} \v$. 
For brevity we suppress the dependence of $\u, \v$ on $\x$ in the rest of the paper.
Note that $\v$ gives the direction in input space that corresponds to the steepest ascent of the linearized network along $\u$. 
See Figure~\ref{fig:ATExplainer} for an illustration of the forward-backward passes through ($\Jac{\x}$ of) the network. 

We can turn this into a regularizer by learning the parameters $\theta$ via
\vspace{-1mm}
\begin{align}
& \min \theta \rightarrow \E_{(\x,y )\sim \hat{P}}\Big[ \ell(y, f(\x)) + \frac{\tilde\lambda}{2} \sigma(\Jac{\x})^2 \Big] \,,
\label{eq:JacobianSN}
\end{align}
where the data-dependent singular value $\sigma(\Jac{\x})$ is computed via Eq.~\ref{eq:datadepspectralnorm}.

By optimality / stationarity\footnote{$\u = \Jac{\x} \v / ||\Jac{\x} \v||_2$. Compare with Eq.~\ref{eq:spectralnormdefinition} and the (2,2)-operator norm ($=$ spectral norm) in Eq.~\ref{eq:operatornorm}.} we also have that 
$\sigma(\Jac{\x}) = \u^\top \Jac{\x} \v 
= || \Jac{\x} \v ||_2$.
Totherther with linearization $\epsilon \Jac{\x}  \v \simeq f(\x + \epsilon \v) - f(\x)$ (which holds with equality if $\x + \epsilon \v \in X(\phi_\x)$), 
we can regularize learning also via
a sum-of-squares based spectral norm regularizer $\frac{\lambda}{2} || f(\x \!+\! \epsilon \v) \!-\! f(\x) ||_2^2$, 
where $\tilde\lambda = \lambda\epsilon^2$.
Both variants can readily be implemented in modern deep learning frameworks. 
We found the sum-of-squares based one to be slightly more numerically stable.

In terms of computational complexity, the data-dependent regularizer is equally expensive as projected gradient ascent based adversarial training, and both are a constant (number of power method iterations) times more expensive than the data-independent variant,
plus an overhead that depends on the batch size, which is mitigated in modern frameworks by parallelizing computations across a batch of data.

\subsection{Data-dependent Operator Norm Regularization}
\label{sec:data-dependentONR}
More generally, we can directly regularize the {data-dependent} $(p, q)$-{operator norm} of the Jacobian. 
To this end, let the input-space (domain) and output-space (co-domain) of the linear map $\Jac{\x}$ be equipped with the $\ell_p$ and $\ell_q$ norms respectively, defined as $|| \x ||_p := (\sum_{i=1}^n | x_i |^p)^{1/p}$ for $1 \leq p < \infty$ and $|| \x ||_\infty := \max_i | x_i |$ for $p=\infty$.
The \textit{data-dependent} $(p, q)$-\textit{operator norm} of $\Jac{\x}$ is defined as
\begin{align}
\label{eq:operatornorm}
|| \Jac{\x} ||_{p, q}  &:=\! 
\max_{ \v : || \v ||_p = 1}  || \Jac{\x} \v ||_q  
\end{align}
which is a data-dependent measure of the maximal amount of signal-gain that can be induced when propagating a norm-bounded input vector through the linearized network.
Table~\ref{tbl:opnorms} shows $(p, q)$-operator norms for typical values of $p$ (domain) and $q$ (co-domain).

For general $(p, q)$-norms, the maximizer $\v$ in Eq.~\ref{eq:operatornorm} can be computed via projected gradient ascent.
Let $\v_0$ be a random vector or an approximation to the maximizer.
The data-dependent $(p, q)$-operator norm of $\Jac{\x}$ can be computed iteratively via
\begin{equation}
\v_k = \Pi_{\{|| \cdot ||_p = 1\}}( \v_{k-1} + \alpha\, \nabla_{\v}  || \Jac{\x} \v_{k-1} ||_q )
\end{equation}
where $\Pi_{\{|| \cdot ||_p = 1\}}(\tilde \v) := \argmin_{\v^* : ||\v^*||_p = 1} || \v^* - \tilde\v ||_2$ is the orthogonal projection onto the $\ell_p$ unit sphere, 
and where $\alpha$ is a step-size or weighting factor, trading off the previous iterate $\v_{k-1}$ with the current gradient step $\nabla_{\v}  || \Jac{\x} \v_{k-1} ||_q$.

By the chain-rule, the computation of the gradient step 
\begin{align}
\label{eq:Jacobiangradientchainrule}
\hspace{-2mm}\nabla_\v || \Jac{\x} \v ||_q 
& = \Jac{\x}^\top \textnormal{sign}(\u) \odot | \u |^{q-1} / || \u ||_q^{q-1} \,\ , \quad \text{where }\ \u = \Jac{\x} \v \,,
\end{align}
can be decomposed into a forward and backward pass through the Jacobian $\Jac{\x}$, 
yielding the following \textit{projected gradient ascent based operator norm iteration method} 
\begin{equation}
\begin{aligned}
\label{eq:datadepoperatornorm}
\hspace{-2mm}\u_{k} &\leftarrow \textnormal{sign}(\tilde{\u}_k) \odot | \tilde{\u}_k |^{q-1} /\, || \tilde{\u}_k ||_q^{q-1}\,, \,\,\, \tilde{\u}_{k} \leftarrow \Jac{\x} \v_{k-1} \\
\hspace{-2mm}\v_k &\leftarrow \Pi_{\{|| \cdot ||_p = 1\}}(\v_{k-1} + \alpha\, \tilde{\v}_k)\, , \quad\quad\,\,\ \tilde{\v}_{k} \leftarrow \Jac{\x}^\top \u_{k} 
\end{aligned}
\end{equation}
where $\odot$\,, $\textnormal{sign}( \cdot )$ and $| \cdot |$ denote elementwise product, sign and absolute-value.
See Figure~\ref{fig:ATExplainer} for an illustration of the forward-backward passes through (the Jacobian $\Jac{\x}$ of) the network.

In particular, we have the following well-known special cases. For $q=2$, the forward pass equation is given by $\u_k \leftarrow \tilde\u_k / || \tilde\u_k ||_2$, while for $q=1$ it is given by $\u_k \leftarrow \textnormal{sign}(\tilde\u_k)$.
The $q=\infty$ limit 
on the other hand is given by
$\lim_{q\to\infty} \textnormal{sign}(\tilde{\u}) \odot | \tilde{\u} |^{q-1} /\, || \tilde{\u} ||_q^{q-1} = |\mathcal{I}|^{-1} \textnormal{sign}(\tilde{\u}) \odot \mathbf{1}_{\mathcal{I}}$
with $\mathbf{1}_{\mathcal{I}} = \sum_{i \in \mathcal{I}} \mathbf{e}_i$, where $\mathcal{I} := \{ j \in [1,...,d] : |\tilde{u}_j| = || \tilde\u ||_\infty \}$ denotes the set of indices at which $\tilde\u$ attains its maximum norm and $\mathbf{e}_i$ is the $i$-th canonical unit vector. See Sec.~\ref{sec:gradientofpnorm} in the  Appendix for a derivation.

It is interesting to note that we can recover the power method update equations by taking the limit $\alpha \to \infty$ (which is well-defined since $\alpha$ is inside the projection) in the above iteration equations, 
which we consider to be an interesting result in its own right, see \textbf{Lemma~\ref{projection_lemma}} in the Appendix.
With this, we obtain the following \textit{power method limit} of the operator norm iteration equations 
\begin{equation}
\begin{aligned}
\label{eq:datadepoperatornorm_powermethodlimit}
\hspace{-2mm}\u_{k} &\leftarrow \textnormal{sign}(\tilde{\u}_k) \odot | \tilde{\u}_k |^{q-1} /\, || \tilde{\u}_k ||_q^{q-1}\ , \,\,\,\,\,\ \tilde{\u}_{k} \leftarrow \Jac{\x} \v_{k-1}  \\
\hspace{-2mm}\v_k &\leftarrow \textnormal{sign}(\tilde{\v}_k) \odot | \tilde{\v}_k |^{p^*-1} /\, || \tilde{\v}_k ||_{p^*}^{p^*-1} \, , \, \tilde{\v}_{k} \leftarrow \Jac{\x}^\top \u_{k} 
\end{aligned}
\end{equation}
The condition that $\alpha \to \infty$ means that in the update equation for $\v_k$  all the weight is placed on the current gradient direction $\tilde\v_k$ whereas no weight is put on the previous iterate $\v_{k-1}$.
See \cite{boyd1974power, higham1992estimating} for a convergence analysis of the power method.

We can turn this into a regularizer by learning the parameters $\theta$ via 
\begin{align}
\label{eq:operatornormregularization}
& \min \theta \rightarrow \E_{(\x,y )\sim \hat{P}}\Big[ \ell(y, f(\x)) +  \tilde\lambda \, || \Jac{\x} ||_{p, q} \Big] 
\end{align}
Note that we can also use the $q$-th power of the operator norm as a regularizer (with a prefactor of $1/q$). 
It is easy to see that this only affects the normalization of $\u_{k}$. 

\subsection{Power Method Formulation of Adversarial Training}
\label{sec:PowerMethodAT}
Adversarial training \cite{goodfellow2014explaining, kurakin2016adversarial, madry2018towards} aims to improve the robustness of a model by training it against an adversary that perturbs each training example subject to a proximity constraint, e.g.\ in $\ell_p$-norm,
\begin{align}
\label{eq:ATobjective}
\hspace{-1mm}& \min \theta \rightarrow \E_{(\x,y )\sim \hat{P}}\Big[ \ell(y, f(\x)) + \lambda \hspace{-1mm}\max_{\x^* \in \mathcal{B}^p_{\epsilon}(\x)} \ell_{\rm adv}(y, f(\x^*)) \Big]
\end{align}
where $\ell_{\rm adv}(\cdot, \cdot)$ denotes the loss function used to find adversarial perturbations (not necessarily the same as the classification loss $\ell(\cdot, \cdot)$).

The adversarial example $\x^*$ is typically computed iteratively, e.g.\ via $\ell_p$-norm constrained projected gradient ascent
(cf.\ the PGD attack in \cite{papernot2017cleverhans} 
or Sec.~3 on PGD-based adversarial attacks in \cite{wong2019wasserstein})
\begin{align}
\x_0  \sim \mathcal{U}( \mathcal{B}^p_{\epsilon}(\x) ) \ , \,\ \x_{k} = \Pi_{\mathcal{B}^p_{\epsilon}(\x) } \Big(  \x_{k-1} + \alpha \hspace{-1mm}\argmax_{\v_k: || \v_k ||_p \leq 1} \v_k^\top \nabla_\x \ell_{\rm adv}(y, f(\x_{k-1})) \Big)
\end{align}
where $\Pi_{\mathcal{B}^p_{\epsilon}(\x) }(\tilde \x) := \argmin_{\x^* \in \mathcal{B}^p_{\epsilon}(\x)} || \x^* \!-\! \tilde\x ||_2$ is the orthogonal projection operator 
into the norm ball $\mathcal{B}^p_{\epsilon}(\x) := \{ \x^* : || \x^* - \x ||_p \leq \epsilon \}$, 
$\alpha$ is a step-size or weighting factor, trading off the previous iterate $\x_{k-1}$ with the current gradient step $\v_k$, and
$y$ is the true or predicted label. 
For targeted attacks the sign in front of $\alpha$ is flipped, so as to descend the loss function into the direction of the target label.

We can in fact derive the following explicit expression for the optimal perturbation to a linear function under an $\ell_p$-norm constraint, 
see \textbf{Lemma~\ref{optimalpert_lemma}} in the Appendix, 
\begin{align}
\label{eq:ATmaximizer}
\v^* = \argmax_{\v : || \v ||_p \leq 1} \v^\top  \z = {\textnormal{sign}(\z) \odot |\z|^{p^*-1}}/{|| \z||_{p^*}^{p^*-1}}
\end{align}
where $\odot$\,, $\textnormal{sign}( \cdot )$ and $| \cdot |$ denote elementwise product, sign and absolute-value,
$p^*$ is the H{\"o}lder conjugate of $p$, given by $1/p + 1/p^* = 1$,
and $\z$ is an arbitrary non-zero vector, 
e.g.\ $\z = \nabla_\x \ell_{\rm adv}(y, f(\x_{}))$.
The derivation can be found in Sec.~\ref{sec:optimallinearpert} in the Appendix.

As a result, $\ell_p$-norm constrained projected gradient ascent can be implemented as follows
\begin{align}
\x_{k} &= \Pi_{\mathcal{B}^p_{\epsilon}(\x) } \Big(  \x_{k-1} + \alpha \,  {\textnormal{sign}(\nabla_\x \ell_{\rm adv}) \odot |\nabla_\x \ell_{\rm adv}|^{{p^*}-1}}/{|| \nabla_\x \ell_{\rm adv}||_{p^*}^{{p^*}-1}} \Big) 
\end{align}
where ${p^*}$ is given by $1/p + 1/{p^*} = 1$ and $\nabla_\x \ell_{\rm adv}$ is short-hand notation for $\nabla_\x \ell_{\rm adv}(y, f(\x_{k-1}))$.

By the chain-rule, the computation of the gradient-step $\tilde{\v}_k := \nabla_\x \ell_{\rm adv}(y, f(\x_{k-1}))$ can be decomposed into a logit-gradient and a Jacobian vector product.
$\ell_p$-norm constrained projected gradient ascent can thus equivalently be written in the following 
\textit{power method like} 
forward-backward pass form (the normalization of $\tilde{\u}_k$ is optional and can be absorbed into the normalization of $\tilde\v_k$)
\vspace{1mm}
\begin{equation}
\begin{aligned}
\label{eq:powermethodformulationofAT}
\u_k &\leftarrow \tilde{\u}_k / ||\tilde{\u}_k ||_2\, ,  \,\,\  \tilde{\u}_k \leftarrow \left. \nabla_\z \ell_{\rm adv}(y, \z) \right|_{\z = f(\x_{k-1})}  \\
\v_k & \leftarrow { \textnormal{sign}(\tilde\v_k) \odot |\tilde\v_k|^{p^*-1} }/{ || \tilde\v_k||_{p^*}^{p^*-1} } ,  \,\,\ \tilde{\v}_k \leftarrow \left. \Jac{\x_{k-1}}^\top \u_k \right._{}  \\
 \x_{k} &\leftarrow \Pi_{\mathcal{B}^p_{\epsilon}(\x) } (  \x_{k-1} + \alpha \v_k ) 
\end{aligned}
\end{equation}
The adversarial loss function determines the logit-space direction $\u_k$ 
in the power method like formulation of adversarial training, 
while $\tilde{\v}_k$ resp.\ $\v_k$ gives the unconstrained resp.\ norm-constrained direction in input space that corresponds to the steepest ascent of the linearized network along $\u_k$.

The corresponding forward-backward pass equations for an $\ell_q$-norm loss on the logits of the clean and perturbed input $\ell_{\rm adv}( f(\x), f(\x^*)) = || f(\x) - f(\x^*) ||_q$ are shown in Eq.~\ref{eq:powermethodformulationofAT_ellq} in the Appendix.

Comparing Eq.~\ref{eq:powermethodformulationofAT} with Eq.~\ref{eq:datadepoperatornorm_powermethodlimit},
we can see that {\it adversarial training is a form of data-dependent operator norm regularization}.
The following theorem states the precise conditions under which they are mathematically equivalent.
The correspondence is proven for $\ell_q$-norm adversarial losses~\cite{sabour2015adversarial, kannan2018adversarial}.
See Sec.~\ref{sec:effectofadversarialloss} in the Appendix for a discussion of the softmax cross-entropy loss.
The proof can be found in Sec.~\ref{sec:proof} in the Appendix.

\begin{tcolorbox}
\begin{theorem}
For $\epsilon$ small enough such that $\mathcal{B}^p_{\epsilon}(\x) \subset X(\phi_\x)$ and in the limit $\alpha \to \infty$, 
$\ell_p$-norm constrained projected gradient ascent based adversarial training
with an $\ell_q$-norm loss on the logits of the clean and perturbed input $\ell_{\rm adv}( f(\x), f(\x^*)) = || f(\x) - f(\x^*) ||_q$, 
with $p,q \in \{1,2,\infty\}$,
is equivalent to the power method limit of data-dependent (p, q)-operator norm regularization of the Jacobian $\Jac{\x}$ of the network. 
\end{theorem}
\end{tcolorbox}

\begin{table}[t]
\centering
\caption{Computing $(p, q)$-operator norms for typical values of $p$ (domain) and $q$ (co-domain). See Sec.~4.3.1 in \cite{tropp2004topics}.
We prove Theorem~1 for all the entries in this table.
In our setting, ``columns'' of $\Jac{\x}$ correspond to paths through the network originating in a specific input neuron, whereas ``rows'' correspond to paths ending in a specific output neuron.}\label{tbl:opnorms}
\vspace{-2mm}
\scriptsize
\vspace{2mm}
\begin{tabular}{cc|cccc}
\toprule
& & \multicolumn{3}{c}{Co-domain} \\[1mm]
\multirow{4}{*}{\begin{turn}{90}\hspace{-1.4cm} Domain \end{turn}}\hspace{-2mm} & &  $\ell_1$   &   $\ell_2$  &  $\ell_\infty$  \\[1mm]
\cmidrule{1-5}
& \multirow{2}{*}{$\ell_1$}     &  max $\ell _{1}$    &  max $\ell _{2}$    &  max $\ell _{\infty }$   \\
& & norm of a column & norm of a column & norm of a column \\[2mm]
& \multirow{2}{*}{$\ell_2$}     & \multirow{2}{*}{NP-hard}   &  \multirow{2}{*}{max singular value}   &  max $\ell _{2}$  \\
& & & &  norm of a row \\[2mm]
& \multirow{2}{*}{$\ell_\infty$}    &  \multirow{2}{*}{NP-hard}   &  \multirow{2}{*}{NP-hard}   &  max $\ell _{1}$ \\
& & & & norm of a row \\
\bottomrule
\end{tabular}
\end{table}


In practice, the correspondence holds \emph{approximately} (to a very good degree) in a region much larger than $X(\phi_\x)$, 
namely as long as the Jacobian of the network remains approximately constant in the uncertainty ball under consideration, see Sec.~\ref{sec:rangeofreg}, specifically Figure~\ref{fig:Linearity} (left), as well as Sec.~\ref{sec:actpatterns}, specifically Figure~\ref{fig:ReluMasks} and Figure~\ref{fig:furtherresultsactivationpatterns}.

\textbf{In summary}, our Theorem confirms that a network’s sensitivity to adversarial examples is characterized through its spectral properties: it is the dominant singular vector (resp.\ the maximizer $\mathbf{v}^*$ in Lemma~\ref{optimalpert_lemma}) corresponding to the largest singular value (resp.\ the $(p, q)$-operator norm) that determines the optimal adversarial perturbation and hence the sensitivity of the model to adversarial examples.

Our results also explain why input gradient regularization and fast gradient method based adversarial training do not sufficiently protect against iterative adversarial attacks, 
namely because the input gradient, resp.\ a single power method iteration, do not yield a sufficiently good approximation for the dominant singular vector in general.
Similarly, we do not expect Frobenius norm ($=$ sum of all singular values) regularization to work as well as data-dependent spectral norm ($=$ largest s.v.) regularization in robustifying against iterative adversarial attacks.
More details in Sec.~\ref{sec:inputgradientregularization}\,\,\&\,\,\ref{sec:frobeniusnormregularization}.


\section{Experimental Results}
\label{sec:experiments}

\subsection{Dataset, Architecture \& Training Methods} 
\label{sec:expsetup}
\vspace{-1mm}
We trained Convolutional Neural Networks (CNNs) with batch normalization on the CIFAR10 data set \cite{krizhevsky2009learning}.
We use a 7-layer CNN as our default platform, since it has good test set accuracy at acceptable computational requirements. For the robustness experiments, we also train a Wide Residual Network (WRN-28-10) \cite{zagoruyko2016wide}. We used an estimated $6$k TitanX GPU hours in total for all our experiments.
Our code is available at \url{https://github.com/yk/neurips20_public}.

We train each classifier with a number of different training methods: 
(i)~`Standard':\ standard empirical risk minimization with a softmax cross-entropy loss, 
(ii)~`Adversarial':\ $\ell_2$- / $\ell_\infty$-norm constrained projected gradient ascent (PGA) based adversarial training,
(iii)~`global SNR':\ global spectral norm regularization {\`a} la Yoshida \& Miyato~\cite{yoshida2017spectral}, 
and (iv) `d.d.\ SNR / ONR':\ data-dependent spectral / operator norm regularization.

As a default attack strategy we use an $\ell_2$- / $\ell_\infty$-norm constrained PGA white-box attack with 10 attack iterations. 
We verified that all our conclusions also hold for larger numbers of attack iterations, however, due to computational constraints we limit to 10.
The attack strength~$\epsilon$ used for training (indicated by a vertical dashed line in the Figures below) was chosen to be the smallest value such that almost all adversarially perturbed inputs to the standard model are successfully misclassified.

The regularization constants were chosen such that the regularized models achieve the same test set accuracy on clean examples as the adversarially trained model does.
Global SNR is implemented with one spectral norm update per training step, as recommended by the authors. 
We provide additional experiments for global SNR with 10 update iterations in Sec.~\ref{sec:furtherresultsglobalsnr} in the Appendix. 
As shown, the 10 iterations make no difference, in line with the regularizer's decoupling from the empirical loss.

Table~\ref{tbl:hypers} in the Appendix summarizes the test set accuracies for the training methods we considered.
Additional experimental results and further details regarding the experimental setup can be found in Secs.~\ref{sec:experimentalsetup} and following in the Appendix.

Shaded areas in the plots below denote standard errors w.r.t.\ the number of test set samples over which the experiment was repeated.


\vspace{-1mm}
\subsection{Spectral Properties}
\label{sec:spectralproperties}
\vspace{-1mm}

\begin{figure}[t!]
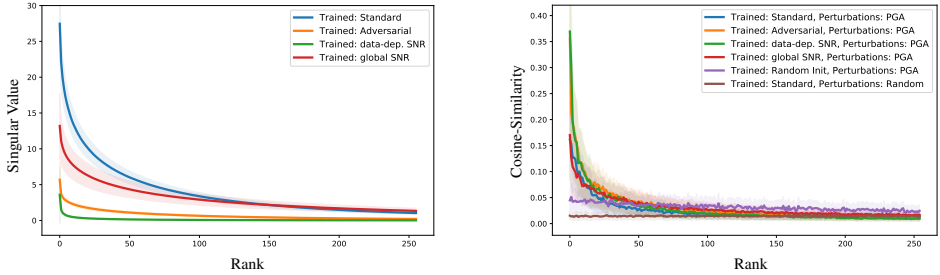

\vspace{-3mm}
\centering
\begin{tabular}{@{}l@{\hspace{1pt}}c@{\hspace{25pt}}l@{\hspace{1pt}}c@{}}
\begin{turn}{90}\hspace{26pt} {\tiny Singular Value} \end{turn} & \adjincludegraphics[width=0.4\linewidth, trim={{0.037\width} {0.07\height} {0.0\width} {0.0\height}},clip]{plots/spectrum} &
\begin{turn}{90}\hspace{24pt} {\tiny Cosine-Similarity} \end{turn} & \adjincludegraphics[width=0.4\linewidth, trim={{0.037\width} {0.07\height} {0.0\width} {0.0\height}},clip]{plots/singular_vectors}  \\[-1mm]
& \quad{\tiny Rank} & & \quad{\tiny Rank} 
\end{tabular}
\vspace{-2mm}
\caption{(Left) Singular value spectrum of the Jacobian $\J_{\phi^{L-1}}(\x)$ for networks $f=\W^L\phi^{L-1}$ trained with different training methods. 
(Right) Alignment of adversarial perturbations with singular vectors $\v_r$ of the Jacobian $\J_{\phi^{L-1}}(\x)$, as a function of the rank $r$ of the singular vector. 
For comparison we also show the cosine-similarity with the singular vectors of a random network. 
We can see that (i) adversarial training and data-dependent spectral norm regularization significantly dampen the singular values, while global spectral norm regularization has almost no effect compared to standard training, and (ii)~adversarial perturbations are strongly aligned with dominant singular vectors.}
\label{fig:Spectrum}
\end{figure}

\textbf{Effect of training method on singular value spectrum.}
We compute the singular value spectrum of the Jacobian $\J_{\phi^{L-1}}(\x)$ for networks $f=\W\phi^{L-1}$ trained with different training methods and evaluated at a number of different test examples. 
Since we are interested in computing the full singular value spectrum, and not just the dominant singular value / vectors as during training, 
using the power method would be too impractical, as it gives us access to only one (the dominant) singular value-vector pair at a time.
Instead, we first extract the Jacobian (which is \textit{per se} defined as a computational graph in modern deep learning frameworks) as an input-dim$\times$output-dim matrix 
and then use available matrix factorization routines to compute the full SVD of the extracted matrix.
For each training method, the procedure is repeated for $200$ randomly chosen clean and corresponding adversarially perturbed test examples.
Further details regarding the Jacobian extraction can be found in Sec.~\ref{sec:extractingJacobian} in the Appendix. 
The results are shown in Figure~\ref{fig:Spectrum} (left). We can see that adversarial training and data-dependent spectral norm regularization significantly dampen the singular values, while global spectral norm regularization has almost no effect compared to standard training.

\textbf{Alignment of adversarial perturbations with singular vectors.}
We compute the cosine-similarity of adversarial perturbations with singular vectors $\v_r$ of the Jacobian $\J_{\phi^{L-1}}(\x)$, 
extracted at a number of test set examples, as a function of the rank of the singular vectors returned by the SVD decomposition. 
For comparison we also show the cosine-similarity with the singular vectors of a random network. 
The results are shown in Figure~\ref{fig:Spectrum} (right).
We can see that for all training methods (except the random network) adversarial perturbations are strongly aligned with the dominant singular vectors
while the alignment decreases with increasing rank. 

Interestingly, this strong alignment with dominant singular vectors also confirms why input gradient regularization and fast gradient method (FGM) based adversarial training 
do not sufficiently protect against iterative adversarial attacks,
namely because the input gradient, resp.\ a single power method iteration, do not yield a sufficiently good approximation for the dominant singular vector in general. See Sec.~\ref{sec:inputgradientregularization} in the Appendix for a more technical explanation.

\subsection{Adversarial Robustness}
\vspace{-1mm}

\textbf{Adversarial classification accuracy.}
A plot of the classification accuracy on adversarially perturbed test examples, as a function of the perturbation strength $\epsilon$, is shown in Figure~\ref{fig:Accuracy}. 
We can see that data-dependent spectral norm regularized models are equally robust to adversarial examples as adversarially trained models and both are significantly more robust than a normally trained one, 
while global spectral norm regularization does not seem to robustify the model substantially. 
This is in line with our earlier observation that adversarial perturbations tend to align with dominant singular vectors
and that they are dampened by adversarial training and data-dependent spectral norm regularization.

\textbf{Interpolating between AT and d.d. SNR.}
We have also conducted an experiment where we train several networks from scratch each with an objective function that convexly combines adversarial training with data-dependent spectral norm regularization in a way that allows us to interpolate between (i) the fraction of adversarial examples relative to clean examples used during adversarial training controlled by $\lambda$ in Eq.~\ref{eq:ATobjective} and (ii) the regularization parameter $\tilde\lambda$ in Eq.~\ref{eq:operatornormregularization}. This allows us to continuously trade-off the contribution of AT with that of d.d. SNR in the empirical risk minimization. 
The results, shown in Figure~\ref{fig:interpolatingATddSNR}, again confirm that the two training methods are equivalent.

\begin{figure}[t]
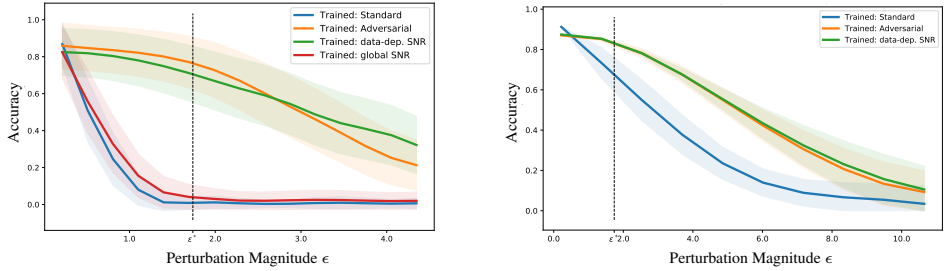

\vspace{-5mm}
\centering
\begin{tabular}{@{}l@{\hspace{1pt}}c@{\hspace{23pt}}l@{\hspace{1pt}}c@{}}
\begin{turn}{90}\hspace{35pt} {\tiny Accuracy} \end{turn} & \adjincludegraphics[width=0.4\linewidth, trim={{0.037\width} {0.085\height} {0.0\width} {0.0\height}},clip]{plots/adversarial_accuracy} &
\begin{turn}{90}\hspace{35pt} {\tiny Accuracy} \end{turn} & \adjincludegraphics[width=0.41\linewidth, trim={{0.037\width} {0.085\height} {0.0\width} {0.065\height}},clip]{plots5/sota.pdf} \\[-1mm]
& \quad{\tiny Perturbation Magnitude $\epsilon$} & & \quad{\tiny Perturbation Magnitude $\epsilon$}
\end{tabular}
\vspace{-2mm}
\caption{
Classification accuracy as a function of perturbation strength $\epsilon$. (Left) 7-layer CNN
(Right) WideResNet WRN-28-10. 
The dashed line indicates the $\epsilon$ used during training.
Curves were aggregated over $2000$ (left) resp. all (right) samples from the test set.
We can see that adversarially trained and data-dependent spectral norm regularized models are equally robust to adversarial attacks.
}
\label{fig:Accuracy}
\end{figure}

\subsection{Local Linearity} 
\label{sec:rangeofreg}
\vspace{-1mm}

\textbf{Validity of linear approximation.}
To determine the range in which the locally linear approximation is valid, we measure the deviation from linearity 
$|| \phi^{L-1}(\x + \z) - ( \phi^{L-1}(\x) + \J_{\phi^{L-1}}(\x) \z ) ||_2$ as the distance $||\z||_2$ is increased in random and adversarial directions $\z$, 
with adversarial perturbations serving as a proxy for the direction in which the linear approximation holds the least.
This allows us to investigate how good the linear approximation for different training methods is, as an increasing number of activation boundaries are crossed with increasing perturbation radius.

\begin{figure}[t]
\centering
\begin{tabular}{@{}l@{\hspace{1pt}}c@{\hspace{25pt}}l@{\hspace{1pt}}c@{}}
\begin{turn}{90}\hspace{14pt}  {\tiny Deviation from Linearity} \end{turn} & \adjincludegraphics[width=0.4\linewidth, trim={{0.042\width} {0.085\height} {0.0\width} {0.0\height}},clip]{plots/linear} &
\begin{turn}{90}\hspace{20pt} {\tiny Top Singular Value} \end{turn} & \adjincludegraphics[width=0.4\linewidth, trim={{0.042\width} {0.085\height} {0.0\width} {0.0\height}},clip]{plots/top_singular_value} \\[-1mm]
& \quad{\tiny Distance from $\x$} & & \quad{\tiny Distance from $\x$} 
\end{tabular}
\vspace{-2mm}
\caption{(Left) Deviation from linearity $|| \phi^{L-1}(\x + \z) - ( \phi^{L-1}(\x) + \J_{\phi^{L-1}}(\x) \z ) ||_2$ as a function of the distance $||\z||_2$ from $\x$ for random and adversarial perturbations $\z$. 
(Right) Largest singular value of the Jacobian $\J_{\phi^{L-1}}(\x\!+\!\z)$ as a function of the magnitude $||\z||_2$. 
We can see that adversarially trained and data-dependent spectral norm regularized models are significantly more linear around data points than the normally trained one
and that the Jacobian $\Jac{\x}$ is a good approximation for the AT and d.-d.\ SNR regularized classifier in the entire $\epsilon^*$-ball around~$\x$.} 
\label{fig:Linearity}
\end{figure}

The results are shown in Figure~\ref{fig:Linearity} (left).
We can see that adversarially trained and data-dependent spectral norm regularized models are significantly more linear than the normally trained one and that 
the Jacobian $\Jac{\x}$ is a good approximation for the AT and d.-d.\ SNR regularized classifier in the entire $\epsilon^*$-ball around~$\x$, 
since the models remain flat even in the adversarial direction for perturbation magnitudes up to the order of the $\epsilon^*$ used during adversarial training (dashed vertical line).
Moreover, since our Theorem is applicable as long as the Jacobian of the network remains (approximately) constant in the uncertainty ball under consideration, this also means that the correspondence between AT and d.d. SNR holds up to the size of the $\epsilon^*$-ball commonly used in AT practice.

\textbf{Largest singular value over distance.}
Figure~\ref{fig:Linearity} (right) shows the largest singular value of the linear operator $\J_{\phi^{L-1}}(\x\!+\!\z)$ as the distance $||\z||_2$ from $\x$ is increased, both along random and adversarial directions $\z$.
We can see that the naturally trained network develops large dominant singular values around the data point,
while the adversarially trained and data-dependent spectral norm regularized models manage to keep the dominant singular value low in the vicinity of $\x$.

\vspace{-2mm}
\subsection{Activation patterns}
\label{sec:actpatterns}
\vspace{-1mm}
An important property of \textit{data-dependent} regularization 
is that it primarily acts on the data manifold whereas it should have comparatively little effect on irrelevant parts of the input space, see Sec.~\ref{sec:globalvslocalregularization}.
We test this hypothesis by comparing activation patterns $\phi_{(\cdot)} \in \{0,1 \}^m$ of perturbed input samples. 
We measure the fraction of shared activations $m^{-1}(\phi_{(\cdot)} \cap \phi_{(\cdot) \!+\! \mathbf{z}})$, 
where $(\cdot)$ is either a data point $\x$ sampled from the test set, shown in Figure~\ref{fig:ReluMasks} (left), 
or a data point $\mathbf{n}$ sampled uniformly from the input domain (a.s.\ not on the data manifold), shown in Figure~\ref{fig:ReluMasks} (right), 
and where $\mathbf{z}$ is a random uniform noise vector of magnitude~$\epsilon$.
From these curves, we can estimate the average size of the ReLU cells making up the activation pattern $\phi_{(\cdot)}$.
We can see that both d.d. SNR and AT significantly increase the size of the ReLU cells around data (in both random and adv.\ directions), thus improving the stability of activation patterns against adversarial examples, yet they have no effect away from the data manifold.
See also Figure~\ref{fig:furtherresultsactivationpatterns} in the Appendix.

\begin{figure}[t!]
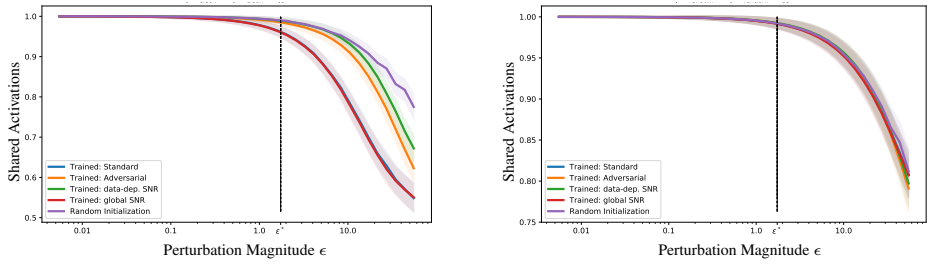

\centering
\begin{tabular}{@{}l@{\hspace{1pt}}c@{\hspace{23pt}}l@{\hspace{1pt}}c@{}}
\begin{turn}{90}\hspace{17pt} {\tiny Shared Activations} \end{turn} & \adjincludegraphics[width=0.4\linewidth, trim={{0.037\width} {0.085\height} {0.0\width} {0.068\height}},clip]{plots5/relu_masks_rand.pdf} &
\begin{turn}{90}\hspace{17pt} {\tiny Shared Activations} \end{turn} & \adjincludegraphics[width=0.4\linewidth, trim={{0.037\width} {0.085\height} {0.0\width} {0.068\height}},clip]{plots5/relu_masks_rinit.pdf} \\[-1mm]
& \quad{\tiny Perturbation Magnitude $\epsilon$} & & \quad{\tiny Perturbation Magnitude $\epsilon$}
\end{tabular}
\vspace{-2mm}
\caption{
Fraction of shared activations as a function of noise magnitude $\epsilon$ 
between activation patterns $\phi_{(\cdot)} \text{ and } \phi_{(\cdot) \!+\! \z}$, 
where $(\cdot)$ is either (left) a data point $\x$ sampled from the test set,
or (right) a data point $\mathbf{n}$ sampled uniformly from the input domain (a.s.\ not on the data manifold) 
and $\z$ is uniform noise of magnitude $\epsilon$.
We can see that both d.d. SNR and AT significantly increase the size of the ReLU cells around data, yet have no effect away from the data manifold.}
\label{fig:ReluMasks}
\end{figure}


\vspace{-1mm}
\section{Conclusion}
\vspace{-2mm}
We established a theoretical link between adversarial training and operator norm regularization for deep neural networks.
Specifically, we derive the precise conditions under which $\ell_p$-norm constrained projected gradient ascent based adversarial training with an $\ell_q$-norm loss on the logits of clean and perturbed inputs is equivalent to data-dependent (p, q) operator norm regularization.
This fundamental connection confirms the long-standing argument that a network's sensitivity to adversarial examples is tied to its spectral properties.
We also conducted extensive empirical evaluations confirming the theoretically predicted effect of adversarial training and data-dependent operator norm regularization on the training of robust classifiers:
(i) adversarial perturbations align with dominant singular vectors,
(ii) adversarial training and data-dependent spectral norm regularization dampen the singular values, and
(iii) both training methods give rise to models that are significantly more linear around data than normally trained ones.

\newpage
\section*{Broader Impact}

The existence of adversarial examples, i.e.\ small perturbations of the input signal, often imperceptible to humans, that are sufficient to induce large changes in the model output, 
poses a real danger when deep neural networks are deployed in the real world,
as potentially safety-critical machine learning systems become vulnerable to attacks that can alter the system's behaviour in malicious ways.
Understanding the origin of this vulnerability and / or acquiring an understanding of how to  robustify deep neural networks against such attacks thus becomes crucial for a safe and responsible deployment of machine learning systems.

{\bf Who may benefit from this research}

Our work contributes to understanding the origin of this vulnerability in that it sheds new light onto the attack algorithms used to find adversarial examples.
It also contributes to building robust machine learning systems in that it allows practitioners to make more informed and well-founded decisions when training robust models.

{\bf Who may be put at a disadvantage from this research}

Our work, like any theoretical work on adversarial examples, may increase the level of understanding of a malevolent person intending to mount adversarial attacks against deployed machine learning systems which may ultimately put the end-users of these systems at risk. 
We would like to note, however, that the attack algorithms we analyze in our work already exist and that we believe that the knowledge gained from our work is more beneficial to making models more robust than it could possibly be used to designing stronger adversarial attacks.

{\bf Consequences of failure of the system}

Our work does not by itself constitute a system of any kind, other than providing a rigorous mathematical framework within which to better understand adversarial robustness.

\begin{ack}
 We would like to thank Michael Tschannen, Sebastian Nowozin and Antonio Orvieto for insightful discussions and helpful comments. All authors are directly funded by ETH Z\"urich.
\end{ack}

\bibliography{bibliography}

\begin{thebibliography}{10}

\bibitem{bartlett2017spectrally}
Peter~L Bartlett, Dylan~J Foster, and Matus~J Telgarsky.
\newblock Spectrally-normalized margin bounds for neural networks.
\newblock In {\em Advances in Neural Information Processing Systems}, pages
  6241--6250, 2017.

\bibitem{bertsimas2018characterization}
Dimitris Bertsimas and Martin~S Copenhaver.
\newblock Characterization of the equivalence of robustification and
  regularization in linear and matrix regression.
\newblock {\em European Journal of Operational Research}, 270(3):931--942,
  2018.

\bibitem{bietti2019kernel}
Alberto Bietti, Gr{\'e}goire Mialon, Dexiong Chen, and Julien Mairal.
\newblock A kernel perspective for regularizing deep neural networks.
\newblock In {\em International Conference on Machine Learning}, pages
  664--674. PMLR, 2019.

\bibitem{biggio2013evasion}
Battista Biggio, Igino Corona, Davide Maiorca, Blaine Nelson, Nedim Srndic,
  Pavel Laskov, Giorgio Giacinto, and Fabio Roli.
\newblock Evasion attacks against machine learning at test time.
\newblock In {\em Joint European conference on machine learning and knowledge
  discovery in databases}, pages 387--402. Springer, 2013.

\bibitem{boyd1974power}
David~W Boyd.
\newblock The power method for lp norms.
\newblock {\em Linear Algebra and its Applications}, 9:95--101, 1974.

\bibitem{bubeck2019adversarial}
S{\'e}bastien Bubeck, Yin~Tat Lee, Eric Price, and Ilya Razenshteyn.
\newblock Adversarial examples from computational constraints.
\newblock In {\em International Conference on Machine Learning}, pages
  831--840, 2019.

\bibitem{carlini2017adversarial}
Nicholas Carlini and David Wagner.
\newblock Adversarial examples are not easily detected: Bypassing ten detection
  methods.
\newblock In {\em Proceedings of the 10th ACM Workshop on Artificial
  Intelligence and Security}, pages 3--14. ACM, 2017.

\bibitem{cisse2017parseval}
Moustapha Cisse, Piotr Bojanowski, Edouard Grave, Yann Dauphin, and Nicolas
  Usunier.
\newblock Parseval networks: Improving robustness to adversarial examples.
\newblock In {\em International Conference on Machine Learning}, pages
  854--863, 2017.

\bibitem{el1997robust}
Laurent El~Ghaoui and Herv{\'e} Lebret.
\newblock Robust solutions to least-squares problems with uncertain data.
\newblock {\em SIAM Journal on matrix analysis and applications},
  18(4):1035--1064, 1997.

\bibitem{farnia2018generalizable}
Farzan Farnia, Jesse~M Zhang, and David Tse.
\newblock Generalizable adversarial training via spectral normalization.
\newblock {\em arXiv preprint arXiv:1811.07457}, 2018.

\bibitem{fawzi2018adversarial}
Alhussein Fawzi, Hamza Fawzi, and Omar Fawzi.
\newblock Adversarial vulnerability for any classifier.
\newblock In {\em Advances in neural information processing systems}, pages
  1178--1187, 2018.

\bibitem{fawzi2018analysis}
Alhussein Fawzi, Omar Fawzi, and Pascal Frossard.
\newblock Analysis of classifiers’ robustness to adversarial perturbations.
\newblock {\em Springer Machine Learning}, 107(3):481--508, 2018.

\bibitem{fawzi2016robustness}
Alhussein Fawzi, Seyed-Mohsen Moosavi-Dezfooli, and Pascal Frossard.
\newblock Robustness of classifiers: from adversarial to random noise.
\newblock In {\em Advances in Neural Information Processing Systems}, pages
  1632--1640, 2016.

\bibitem{feinman2017detecting}
Reuben Feinman, Ryan~R Curtin, Saurabh Shintre, and Andrew~B Gardner.
\newblock Detecting adversarial samples from artifacts.
\newblock {\em arXiv preprint arXiv:1703.00410}, 2017.

\bibitem{gao2016distributionally}
Rui Gao and Anton~J Kleywegt.
\newblock Distributionally robust stochastic optimization with wasserstein
  distance.
\newblock {\em arXiv preprint arXiv:1604.02199}, 2016.

\bibitem{gilmer2018adversarial}
Justin Gilmer, Luke Metz, Fartash Faghri, Samuel~S Schoenholz, Maithra Raghu,
  Martin Wattenberg, and Ian Goodfellow.
\newblock Adversarial spheres.
\newblock {\em arXiv preprint arXiv:1801.02774}, 2018.

\bibitem{goodfellow2014explaining}
Ian~J Goodfellow, Jonathon Shlens, and Christian Szegedy.
\newblock Explaining and harnessing adversarial examples.
\newblock {\em arXiv preprint arXiv:1412.6572}, 2014.

\bibitem{grosse2017statistical}
Kathrin Grosse, Praveen Manoharan, Nicolas Papernot, Michael Backes, and
  Patrick McDaniel.
\newblock On the (statistical) detection of adversarial examples.
\newblock {\em arXiv preprint arXiv:1702.06280}, 2017.

\bibitem{gu2014towards}
Shixiang Gu and Luca Rigazio.
\newblock Towards deep neural network architectures robust to adversarial
  examples.
\newblock {\em arXiv preprint arXiv:1412.5068}, 2014.

\bibitem{hein2017formal}
Matthias Hein and Maksym Andriushchenko.
\newblock Formal guarantees on the robustness of a classifier against
  adversarial manipulation.
\newblock In {\em Advances in Neural Information Processing Systems}, pages
  2266--2276, 2017.

\bibitem{higham1992estimating}
Nicholas~J Higham.
\newblock Estimating the matrixp-norm.
\newblock {\em Numerische Mathematik}, 62(1):539--555, 1992.

\bibitem{kannan2018adversarial}
Harini Kannan, Alexey Kurakin, and Ian Goodfellow.
\newblock Adversarial logit pairing.
\newblock {\em arXiv preprint arXiv:1803.06373}, 2018.

\bibitem{krizhevsky2009learning}
Alex Krizhevsky and Geoffrey Hinton.
\newblock Learning multiple layers of features from tiny images.
\newblock 2009.

\bibitem{kurakin2016adversarial}
Alexey Kurakin, Ian Goodfellow, and Samy Bengio.
\newblock Adversarial examples in the physical world.
\newblock {\em arXiv preprint arXiv:1607.02533}, 2016.

\bibitem{lyu2015unified}
Chunchuan Lyu, Kaizhu Huang, and Hai-Ning Liang.
\newblock A unified gradient regularization family for adversarial examples.
\newblock In {\em 2015 IEEE International Conference on Data Mining}, pages
  301--309. IEEE, 2015.

\bibitem{madry2018towards}
Aleksander Madry, Aleksandar Makelov, Ludwig Schmidt, Dimitris Tsipras, and
  Adrian Vladu.
\newblock Towards deep learning models resistant to adversarial attacks.
\newblock In {\em International Conference on Learning Representations}, 2018.

\bibitem{metzen2017detecting}
Jan~Hendrik Metzen, Tim Genewein, Volker Fischer, and Bastian Bischoff.
\newblock On detecting adversarial perturbations.
\newblock {\em arXiv preprint arXiv:1702.04267}, 2017.

\bibitem{miyato2018spectral}
Takeru Miyato, Toshiki Kataoka, Masanori Koyama, and Yuichi Yoshida.
\newblock Spectral normalization for generative adversarial networks.
\newblock In {\em International Conference on Learning Representations}, 2018.

\bibitem{miyato2018virtual}
Takeru Miyato, Shin-ichi Maeda, Masanori Koyama, and Shin Ishii.
\newblock Virtual adversarial training: a regularization method for supervised
  and semi-supervised learning.
\newblock {\em IEEE transactions on pattern analysis and machine intelligence},
  41(8):1979--1993, 2018.

\bibitem{miyato2015distributional}
Takeru Miyato, Shin-ichi Maeda, Masanori Koyama, Ken Nakae, and Shin Ishii.
\newblock Distributional smoothing with virtual adversarial training.
\newblock {\em arXiv preprint arXiv:1507.00677}, 2015.

\bibitem{moosavi2016deepfool}
Seyed~Mohsen Moosavi~Dezfooli, Alhussein Fawzi, and Pascal Frossard.
\newblock Deepfool: a simple and accurate method to fool deep neural networks.
\newblock In {\em Proceedings of 2016 IEEE Conference on Computer Vision and
  Pattern Recognition (CVPR)}, number EPFL-CONF-218057, 2016.

\bibitem{namkoong2017variance}
Hongseok Namkoong and John~C Duchi.
\newblock Variance-based regularization with convex objectives.
\newblock In {\em Advances in Neural Information Processing Systems}, pages
  2975--2984, 2017.

\bibitem{novak2018sensitivity}
Roman Novak, Yasaman Bahri, Daniel~A Abolafia, Jeffrey Pennington, and Jascha
  Sohl-Dickstein.
\newblock Sensitivity and generalization in neural networks: an empirical
  study.
\newblock In {\em International Conference on Learning Representations}, 2018.

\bibitem{papernot2017cleverhans}
Nicolas Papernot, Nicholas Carlini, Ian Goodfellow, Reuben Feinman, Fartash
  Faghri, Alexander Matyasko, Karen Hambardzumyan, Yi-Lin Juang, Alexey
  Kurakin, Ryan Sheatsley, Abhibhav Garg, and Yen-Chen Lin.
\newblock cleverhans v2.0.0: an adversarial machine learning library.
\newblock {\em arXiv preprint arXiv:1610.00768}, 2017.

\bibitem{papernot2016transferability}
Nicolas Papernot, Patrick McDaniel, and Ian Goodfellow.
\newblock Transferability in machine learning: from phenomena to black-box
  attacks using adversarial samples.
\newblock {\em arXiv preprint arXiv:1605.07277}, 2016.

\bibitem{raghu2017expressive}
Maithra Raghu, Ben Poole, Jon Kleinberg, Surya Ganguli, and Jascha~Sohl
  Dickstein.
\newblock On the expressive power of deep neural networks.
\newblock In {\em Proceedings of the 34th International Conference on Machine
  Learning-Volume 70}, pages 2847--2854. JMLR. org, 2017.

\bibitem{raghunathan2018certified}
Aditi Raghunathan, Jacob Steinhardt, and Percy Liang.
\newblock Certified defenses against adversarial examples.
\newblock In {\em International Conference on Learning Representations}, 2018.

\bibitem{roth2019odds}
Kevin Roth, Yannic Kilcher, and Thomas Hofmann.
\newblock The odds are odd: A statistical test for detecting adversarial
  examples.
\newblock In {\em International Conference on Machine Learning}, pages
  5498--5507, 2019.

\bibitem{sabour2015adversarial}
Sara Sabour, Yanshuai Cao, Fartash Faghri, and David~J Fleet.
\newblock Adversarial manipulation of deep representations.
\newblock {\em arXiv preprint arXiv:1511.05122}, 2015.

\bibitem{schmidt2018adversarially}
Ludwig Schmidt, Shibani Santurkar, Dimitris Tsipras, Kunal Talwar, and
  Aleksander Madry.
\newblock Adversarially robust generalization requires more data.
\newblock In {\em Advances in Neural Information Processing Systems}, pages
  5014--5026, 2018.

\bibitem{shaham2018understanding}
Uri Shaham, Yutaro Yamada, and Sahand Negahban.
\newblock Understanding adversarial training: Increasing local stability of
  supervised models through robust optimization.
\newblock {\em Neurocomputing}, 307:195--204, 2018.

\bibitem{sinha2018certifying}
Aman Sinha, Hongseok Namkoong, and John Duchi.
\newblock Certifying some distributional robustness with principled adversarial
  training.
\newblock In {\em International Conference on Learning Representations}, 2018.

\bibitem{szegedy2013intriguing}
Christian Szegedy, Wojciech Zaremba, Ilya Sutskever, Joan Bruna, Dumitru Erhan,
  Ian Goodfellow, and Rob Fergus.
\newblock Intriguing properties of neural networks.
\newblock {\em arXiv preprint arXiv:1312.6199}, 2013.

\bibitem{tanay2016boundary}
Thomas Tanay and Lewis Griffin.
\newblock A boundary tilting persepective on the phenomenon of adversarial
  examples.
\newblock {\em arXiv preprint arXiv:1608.07690}, 2016.

\bibitem{tropp2004topics}
Joel~Aaron Tropp.
\newblock {\em Topics in sparse approximation}.
\newblock PhD thesis, 2004.

\bibitem{tsuzuku2018lipschitz}
Yusuke Tsuzuku, Issei Sato, and Masashi Sugiyama.
\newblock Lipschitz-margin training: Scalable certification of perturbation
  invariance for deep neural networks.
\newblock In {\em Advances in Neural Information Processing Systems}, pages
  6541--6550, 2018.

\bibitem{wong2019wasserstein}
Eric Wong, Frank Schmidt, and Zico Kolter.
\newblock Wasserstein adversarial examples via projected sinkhorn iterations.
\newblock In {\em International Conference on Machine Learning}, pages
  6808--6817, 2019.

\bibitem{xu2009robustness}
Huan Xu, Constantine Caramanis, and Shie Mannor.
\newblock Robustness and regularization of support vector machines.
\newblock {\em Journal of Machine Learning Research}, 10(Jul):1485--1510, 2009.

\bibitem{xu2017feature}
Weilin Xu, David Evans, and Yanjun Qi.
\newblock Feature squeezing: Detecting adversarial examples in deep neural
  networks.
\newblock {\em arXiv preprint arXiv:1704.01155}, 2017.

\bibitem{yoshida2017spectral}
Yuichi Yoshida and Takeru Miyato.
\newblock Spectral norm regularization for improving the generalizability of
  deep learning.
\newblock {\em arXiv preprint arXiv:1705.10941}, 2017.

\bibitem{zagoruyko2016wide}
Sergey Zagoruyko and Nikos Komodakis.
\newblock Wide residual networks.
\newblock {\em arXiv preprint arXiv:1605.07146}, 2016.

\end{thebibliography}
\bibliographystyle{plain}


\clearpage\newpage
\section{Appendix}

We begin with a short recap on robust optimization in linear regression in \textbf{Section~\ref{sec:ROforLinearRegression}}.
In \textbf{Section~\ref{sec:inputgradientregularization}} we lay out why input gradient regularization and fast gradient method (FGM) based adversarial training cannot in general effectively robustify against iterative adversarial attacks.
Similarly, in \textbf{Section~\ref{sec:frobeniusnormregularization}} we argue why we do not expect Frobenius norm regularization to work as well as data-dependent spectral norm regularization in robustifying against $\ell_2$-norm bounded iterative adversarial attacks.
In \textbf{Section~\ref{sec:effectofadversarialloss}} we analyze the power method like formulation of adversarial training for the softmax cross-entropy loss. 
The proof of our main Theorem and the corresponding Lemmas can be found in \textbf{Sections~\ref{sec:gradientofpnorm}\,-\,\ref{sec:proof}}. 
Additional implementation details can be found in \textbf{Sections~\ref{sec:extractingJacobian}\,-\,\ref{sec:hyperparamsweep}}.
Additional experimental results are presented from \textbf{Section~\ref{sec:furtherresultslargealpha}} on. 

\subsection{Recap: Robust Optimization and Regularization for Linear Regression}
\label{sec:ROforLinearRegression}
In this section, we recapitulate the basic ideas on the relation between robust optimization and regularization presented in \cite{bertsimas2018characterization}.
Note that the notation deviates slightly from the main text: most importantly, the perturbations $\triangle$ refer to perturbations of the entire training data $\mX$, as is common in robust optimization.

Consider linear regression with additive perturbations $\triangle$ of the data matrix $\mX$
\begin{align}
\min_\w  \max_{\triangle \in \uSet} h\left( \y - (\mX + \triangle) \w \right),  
\label{eq:robust}
\end{align} 
where $h: \Re^n \to \Re$ denotes a loss function and $\uSet$ denotes the uncertainty set.
A general way to construct $\uSet$ is as a ball of bounded matrix norm perturbations $\uSet = \{ \triangle: \| \triangle \| \le \lambda \}$. Of particular interest are induced matrix norms
\begin{align}
\| \mathbf A \|_{g, h}  := \max_{\w} \bigg\{ \frac{h(\mathbf A \w)}{g(\w)} \bigg\}, 
\end{align}
where $h: \Re^n \to \Re$ is a semi-norm and $g: \Re^d \to \Re$ is a norm.
It is obvious that if $h$ fulfills the triangle inequality then one can upper bound 
\begin{equation}
\begin{aligned}
\label{eq:upper}
h\left( \y - (\mX + \triangle) \w \right) &
 \le h(\y - \mX\w) + h (\triangle \w) \\
& \le h(\y - \mX\w) + \lambda \, g(\w)\,, \quad \forall \triangle \in \uSet \,, 
\end{aligned}
\end{equation}
by using (a) the triangle inequality and (b) the definition of the matrix norm.

The question then is, under which circumstances both inequalities become equalities at the maximizing $\triangle^*$. 
It is straightforward to check \cite{bertsimas2018characterization} Theorem~1 that specifically we may choose the rank $1$ matrix
\begin{align}
\triangle^* = \frac{\lambda}{h(\mathbf r)} \mathbf r \v^\top, 
\end{align}
where 
\begin{align}
\mathbf r = \y - \mX\w \,\, , \,\,\ \v = \argmax_{\v: g^*(\v)=1} \left\{ \v^\top \w \right\}, 
\end{align}
with $g^*$ as the dual norm.
If $h(\mathbf r)=0$ then one can pick any $\mathbf u$ for which $h(\mathbf u)=1$ to form $\triangle = \lambda \mathbf u \v^\top$ 
(such a $\mathbf u$ has to exist if $h$ is not identically zero).
This shows that, for robust linear regression with induced matrix norm uncertainty sets,  robust optimization is equivalent to regularization.

\subsection{On Input Gradient Regularization and Adversarial Robustness}
\label{sec:inputgradientregularization}

In this Section we will lay out why input gradient regularization and fast-gradient method (FGM) based adversarial training cannot in general effectively robustify against iterative adversarial attacks. 

In Section~\ref{sec:spectralproperties} ``Alignment of adversarial perturbations with singular vectors'', we have seen that perturbations of iterative adversarial attacks strongly align with the dominant right singular vectors $\v$ of the Jacobian $\Jac{\x}$, see Figure~\ref{fig:Spectrum} (right). 
This alignment reflects also what we would expect from theory, since the dominant right singular vector $\v$ precisely defines the direction in input space along which a norm-bounded perturbation induces the maximal amount of signal-gain when propagated through the linearized network, see comment after Equation~\ref{eq:operatornorm}.

Interestingly, this tendency to align with dominant singular vectors  explains why input gradient regularization and fast gradient method (FGM) based adversarial training 
do not sufficiently protect against iterative adversarial attacks,
namely because the input gradient, resp.\ a single power method iteration, do not yield a sufficiently good approximation for the dominant singular vector in general.

In short, data-dependent operator norm regularization and iteartive attacks based adversarial training correspond to multiple forward-backward passes through (the Jacobian of) the network, while input gradient regularization and FGM based adversarial training corresponds to just a single forward-backward pass. 

More technically, the right singular vector $\v$ gives the direction in input space that corresponds to the steepest ascent of $f(\x)$ \textit{along} the left singular vector $\u$. In input gradient regularization, the logit space direction $\u$ is determined by the loss function (see Section~\ref{sec:effectofadversarialloss} for an example using the softmax cross-entropy loss), which in general is however neither equal nor a good enough approximation to the dominant left singular vector $\u$ of $\Jac{\x}$.

In other words, if we knew the dominant singular vector $\u$ of $\Jac{\x}$, we could compute the direction $\v$ in a single backward-pass. The computation of the dominant singular vector $\u$, however, involves multiple power-method rounds of forward-backward propagation through $\Jac{\x}$ in general.

\subsection{On Frobenius Norm Regularization and Adversarial Robustness}
\label{sec:frobeniusnormregularization}

In this Section, we briefly contrast data-dependent spectral norm regularization with Frobenius norm regularization.

As we have seen in Section~\ref{sec:data-dependentONR}, it is the dominant singular vector  corresponding to the largest singular value that determines the optimal adversarial perturbation to the Jacobian and hence the maximal amount of signal-gain that can be induced when propagating an $\ell_2$-norm bounded input vector through the linearized network. Writing $\Jac{\x} = \sigma_1 \u_1 \v_1^T + \sigma_2 \u_2 \v_2^T + …$ in SVD form, it is clear that the largest change in output for a given change in input aligns with $\v_1$. This crucial fact is only indirectly captured by regularizing the Frobenius norm, in that the Frobenius norm ($=$ sum of all singular values) is a trivial upper bound on the spectral norm ($=$ largest singular value). 

For that reason, we do not expect the Frobenius norm to work as well as data-dependent spectral norm regularization in robustifying against $\ell_2$-norm bounded iterative adversarial attacks.

\subsection{Adversarial Training with Cross-Entropy Loss}
\label{sec:effectofadversarialloss}
The adversarial loss function determines the logit-space direction $\u_k$ in the power method like formulation of adversarial training in Equation~\ref{eq:powermethodformulationofAT}.

Let us consider this for the softmax cross-entropy loss, defined as $\ell_{\rm adv}(y,\z) := - \log(s_y(\z))$,
\begin{align}
\ell_{\rm adv}(y,\z) = - \z_y + \log\Big( \sum_{k=1}^d \exp(\z_k) \Big) 
\end{align}
where the softmax is given by
\begin{align}
\ s_y(\z) &:=  \frac{\exp(\z_y)}{\sum_{k=1}^d \exp(\z_k)}
\end{align}
Untargeted $\ell_2$-PGA (forward pass)
\begin{align}
\hspace{-3mm}\left[ \tilde{\u}_k \leftarrow \left. \nabla_\z \ell_{\rm adv}(y,\z) \right|_{\z=f(\x_{k-1})} \right]_i = s_i(f(\x_{k-1})) - \delta_{iy}
\end{align}
Targeted $\ell_2$-PGA (forward pass)
\begin{align}
\left[ \tilde{\u}_k \leftarrow \left. -\nabla_\z \ell_{\rm adv}(y_{\rm adv},\z) \right|_{\z=f(\x_{k-1})} \right]_i = \delta_{i y_{\rm adv}} -  s_i(f(\x_{k-1})) 
\end{align}
Notice that the logit gradient can be computed in a forward pass by analytically expressing it in terms of the arguments of the loss function.

\newpage
Interestingly, for a temperature-dependent softmax cross-entropy loss, the logit-space direction becomes a ``label-flip'' vector in the low-temperature limit (high inverse temperature $\beta \to \infty$) 
where the softmax $s^\beta_y(\z) :=  \exp(\beta \z_y) / (\sum_{k=1}^d \exp(\beta \z_k) )$ converges to the argmax: $s^\beta(\z) \overset{\beta\to\infty}{\longrightarrow} \arg \max(\z)$.
E.g.\ for targeted attacks $\big[ \u_k^{\beta\to\infty} \big]_i = \delta_{i y_{\rm adv}} - \delta_{i y(\x_{k-1})}$.
This implies that in the high $\beta$ limit, iterative PGA finds an input space perturbation $\v_k$ that corresponds to the steepest ascent of $f$ along the ``label flip'' direction $\u_k^{\beta \to \infty}$.

\textbf{A note on canonical link functions.}
The gradient of the loss w.r.t.\ the logits of the classifier takes the form ``prediction - target'' for both the sum-of-squares error as well as the softmax cross-entropy loss. This is in fact a general result of modelling the target variable with a conditional distribution from the exponential family along with a canonical activation function.
This means that adversarial attacks try to find perturbations in input space that induce a logit perturbation that aligns with the difference between the current prediction and the attack target.

\bigskip
\subsection{Gradient of p-Norm}
\label{sec:gradientofpnorm}
The gradient of any p-norm is given by
\begin{align}
\nabla_\x || \x ||_p = \frac{\textnormal{sign}(\x) \odot | \x |^{p-1}}{|| \x ||_p^{p-1}}
\end{align}
where $\odot$\,, $\textnormal{sign}( \cdot )$ and $| \cdot |$ denote elementwise product, sign and absolute value.

In this section, we take a closer look at the $p\to\infty$ limit,
\begin{align}
\hspace{-2mm}\lim_{p\to\infty} \nabla_\x || \x ||_p  = |\mathcal{I}|^{-1} \textnormal{sign}(\x) \odot \mathbf{1}_{\mathcal{I}}
\end{align}
with $\mathbf{1}_{\mathcal{I}} = \sum_{i \in \mathcal{I}} \mathbf{e}_i$, where $\mathcal{I} := \{ j \in [1,...,n] : |x_j| = || \x ||_\infty \}$ denotes the set of indices at which $\x$ attains its maximum norm and $\mathbf{e}_i$ is the $i$-th canonical unit vector.

The derivation goes as follows. Consider
\begin{align}
\lim_{p\to\infty} \sum_{i} \frac{| x_i |^p}{ || \x ||_\infty^p} =  \lim_{p\to\infty} \sum_{i} \Big( \frac{| x_i |}{ || \x ||_\infty} \Big)^p = | \mathcal{I}|
\end{align}
Thus
\begin{align}
 \frac{|| \x ||_p^{p-1}}{ || \x ||_\infty^{p-1}} = \Big( \frac{\sum_i | x_i |^p}{ || \x ||_\infty^p} \Big)^{(p-1)/p} \overset{p\to\infty}{\longrightarrow} | \mathcal{I}|
\end{align}
since $| \mathcal{I}|^{(p-1)/p} \overset{p\to\infty}{\longrightarrow} | \mathcal{I}|$.
Now consider
\begin{align}
 \frac{ \textnormal{sign}(x_i) | x_i |^{p-1}}{ || \x ||_p^{p-1}} = \frac{ \textnormal{sign}(x_i) | x_i |^{p-1} / || \x ||_\infty^{p-1}}{ || \x ||_p^{p-1} / || \x ||_\infty^{p-1}} 
\end{align}
The numerator 
\begin{align}
\hspace{-3mm}\textnormal{sign}(x_i) | x_i |^{p-1} / || \x ||_\infty^{p-1} \overset{p\to\infty}{\longrightarrow} 
\begin{cases}
0 \quad \text{if} \quad i \notin \mathcal{I} \\
\textnormal{sign}(x_i) \,\ \text{if} \,\ i \in \mathcal{I}
\end{cases}
\end{align}
The denominator $|| \x ||_p^{p-1} / || \x ||_\infty^{p-1} \overset{p\to\infty}{\longrightarrow} | \mathcal{I} |$.
The rest is clever notation.

\newpage
\subsection{Optimal Perturbation to Linear Function}
\label{sec:optimallinearpert}

\begin{lemma}[Optimal Perturbation]\label{optimalpert_lemma}
Explicit expression for the optimal perturbation to a linear function under an $\ell_p$-norm constraint\footnote{Note that for $p\in\{1,\infty\}$ the maximizer might not be unique, in which case we simply choose a specific representative.}.
Let $\z$ be an arbitrary non-zero vector, e.g.\ $\z = \nabla_\x \ell_{\rm adv}$, and let $\v$ be a vector of the same dimension. Then,
\begin{align}
\v^* = \argmax_{\v : || \v ||_p \leq 1} \v^\top  \z = \frac{\textnormal{sign}(\z) \odot |\z|^{p^*-1}}{|| \z||_{p^*}^{p^*-1}}
\end{align}
where $\odot$\,, $\textnormal{sign}( \cdot )$ and $| \cdot |$ denote elementwise product, sign and absolute value,
and $p^*$ is the H{\"o}lder conjugate of $p$, given by $1/p + 1/p^* = 1$.
Note that the maximizer $\v^*$ is attained at a $\v$ with $|| \v ||_p = 1$, since $\v^\top \z$ is linear in $\v$.
\end{lemma}

In particular, we have the following special cases, which is also how the Projected Gradient Descent attack is implemented in the cleverhans library \cite{papernot2017cleverhans}, 
\begin{equation}
\label{eq:optimalperturbationspecialcases}
\begin{aligned}
\hspace{-1mm}\argmax_{\v : || \v ||_p \leq 1} \v^\top  \z =
\begin{cases}
\ \z / || \z ||_2 \hspace{19mm}\, \text{for} \,\ p=2 \\[1mm]
\ \textnormal{sign}(\z) \hspace{19mm}\, \text{for} \,\ p=\infty \\[1mm]
\ |\mathcal{I}|^{-1} \textnormal{sign}(\z) \odot \mathbf{1}_{\mathcal{I}} \quad \text{for} \,\ p=1
\end{cases}
\end{aligned}
\end{equation}
The optimal $\ell_1$-norm constrained perturbation $\lim_{p^*\to\infty} \textnormal{sign}(\z) \odot | \z |^{p^*-1} /\, || \z ||_{p^*}^{p^*-1}$ can be taken to be 
$|\mathcal{I}|^{-1} \textnormal{sign}(\z) \odot \mathbf{1}_{\mathcal{I}}$
with $\mathbf{1}_{\mathcal{I}} = \sum_{i \in \mathcal{I}} \mathbf{e}_i$, where $\mathcal{I} := \{ j \in [1,...,n] : |z_j| = || \z ||_\infty \}$ denotes the set of indices at which $\z$ attains its maximum norm, and $\mathbf{e}_i$ is the $i$-th canonical unit-vector. Note that any other convex combination of unit-vectors from the set of indices at which $\z$ attains its maximum absolute value is also a valid optimal $\ell_1$-norm constrained perturbation. 

Finally, before we continue with the proof, we would like to note that the above expression for the maximizer has already been stated in \cite{moosavi2016deepfool}, although it hasn't been derived there.

\begin{proof}
By H{\"o}lder's inequality, we have for non-zero $\z$,
\begin{align}
\v^\top \z \leq || \v ||_p || \z ||_{p^*}\,\ , \,\ \text{w.\ equality iff}\,\ | \v |^p = \gamma | \z |^{p^*}
\end{align}
i.e.\ equality\footnote{Note that technically equality only holds for $p, {p^*} \in (1, \infty)$. But one can easily check that the explicit expressions for $p\in\{1,\infty\}$ are in fact optimal. See comment after Equation 1.1 in \cite{higham1992estimating}.} 
holds if and only if $| \v |^p$ and $| \z |^{p^*}$ are linearly dependent, where $| \cdot |$ denotes elementwise absolute-value, and where 
${p^*}$ is given by $1/p + 1/{p^*} = 1$.
The proportionality constant $\gamma$ is determined by the normalization requirement $|| \v ||_p = 1$. 
For $1 < p < \infty$, we have
\begin{align}
& || \v ||_p = ( \gamma \sum_{i=1}^n | z_i |^{p^*})^{1/p} \overset{!}=1  \implies \gamma = || \z ||_{p^*}^{-{p^*}} 
\end{align}
Thus, equality holds iff $| \v | = | \z |^{{p^*}/p} / || \z ||_{p^*}^{{p^*}/p}$, 
which implies that $\v^* = \textnormal{sign}(\z) \odot |\z|^{p^*-1} /|| \z||_{p^*}^{p^*-1}$,
since $\v$ must have the same $\textnormal{sign}$ as $\z$ and ${p^*}/p = {p^*}-1$.
For $p=1$, $\gamma = (| \mathcal{I} |\, || \z ||_\infty)^{-1}$, where $\mathcal{I} := \{ j \in [1,...,n] : | {z}_j| = || \z ||_\infty \}$, i.e.\ $| \mathcal{I} |$ counts the multiplicity of the maximum element in the maximum norm. 
It is easy to see that $|\mathcal{I}|^{-1} \textnormal{sign}(\z) \odot \mathbf{1}_{\mathcal{I}}$ is a maximizer in this case. 
For $p=\infty$, it is trivially clear that the maximizer $\v^*$ is given by $\textnormal{sign}(\z)$.
\end{proof}

\newpage
As an illustrative sanity-check, let us also verify that the above explicit expression for the optimal perturbation $\v^*$ has $\ell_p$ norm equal to one. Let $\z$ be an arbitrary non-zero vector of the same dimension as $\v$ and let $1 < p < \infty$, then
\begin{align}
|| \v^* ||_p &= \Big(\sum_{i=1}^n \Big|\frac{\textnormal{sign}(z_i)\,  |z_i|^{p^*-1}}{|| \z||_{p^*}^{p^*-1}}\Big|^{p}\Big)^{1/p} \\
&= \Big(\sum_{i=1}^n \Big(\frac{|z_i|^{p^*-1}}{|| \z||_{p^*}^{p^*-1}}\Big)^{p}\Big)^{1/p} \\
&= \Big(\sum_{i=1}^n \frac{|z_i|^{p^*}}{|| \z||_{p^*}^{p^*}}\Big)^{1/p}  = 1
\end{align}
where we have used that ($p^*-1)p = p^*$.
For $p = 1$, we have
\begin{align}
|| \v^* ||_1 &= \sum_{i=1}^n \Big| \frac{1}{|\mathcal{I}|} \textnormal{sign}(z_i)  \mathds{1}_{\{i \in \mathcal{I}\}} \Big| \\
&= \sum_{i=1}^n \Big| \frac{1}{|\mathcal{I}|} \mathds{1}_{\{i \in \mathcal{I}\}} \Big| = 1
\end{align}
where $\mathds{1}_{\{ \cdot \}}$ denotes the indicator function.
For $p = \infty$, we have
\begin{align}
|| \v^* ||_\infty &= \max_i |\textnormal{sign}(z_i) |  = 1
\end{align}

\bigskip
\subsection{Projection Lemma}
\label{sec:projectionlemma}
In this section, we prove the following intuitive Lemma.
\begin{lemma}[Projection Lemma]\label{projection_lemma}
(First part) Let $\v$ and \mbox{$\tilde\v\neq\mathbf{0}$} be two arbitrary (non-zero) vectors of the same dimension. Then
\begin{align}
\lim_{\alpha\to\infty} \Pi_{\{|| \cdot ||_p = 1\}}(\v +\alpha \tilde\v) = \frac{ \textnormal{sign}(\tilde\v) \odot |\tilde\v|^{p^*-1} }{ || \tilde\v ||_{p^*}^{p^*-1} }
\end{align}
where $\odot$\,, $\textnormal{sign}( \cdot )$ and $| \cdot |$ denote elementwise product, sign and absolute-value, and 
where $p^*$ denotes the H{\"o}lder conjugate of $p$ defined by $1/p + 1/p^* = 1$.
Moreover, if $\tilde\v$ is of the form $\tilde\v = \textnormal{sign}(\z) \odot |\z|^{p^*-1} / || \z ||_{p^*}^{p^*-1}$, for an arbitrary non-zero vector $\z$ and $p \in \{1,2,\infty\}$, then $\lim_{\alpha\to\infty} \Pi_{\{|| \cdot ||_p = 1\}}(\v +\alpha \tilde\v) = \tilde\v$. 
(Second part) Let $\x_{k-1} \in \mathcal{B}^p_{\epsilon}(\x)$ and let $\v_k\neq\mathbf{0}$ be an arbitrary non-zero vector of the same dimension.
Then
\begin{align}
\x_k & = \lim_{\alpha \to \infty} \Pi_{\mathcal{B}^p_{\epsilon}(\x) } (  \x_{k-1} + \alpha \v_k ) \\
&= \x + \lim_{\alpha \to \infty} \Pi_{\{|| \cdot ||_p = \epsilon\}}(\alpha {\v}_k) \\
&= \x + \epsilon\, \frac{ \textnormal{sign}(\v_k) \odot |\v_k|^{p^*-1} }{ || \v_k ||_{p^*}^{p^*-1} }
\end{align}
Moreover, if $\v_k$ is of the form $\v_k = \textnormal{sign}(\z) \odot |\z|^{p^*-1} / || \z ||_{p^*}^{p^*-1}$, for $p \in \{1,2,\infty\}$ and an arbitrary non-zero vector $\z$ (as is the case if $\v_k$ is given by the backward pass in Equation~\ref{eq:powermethodformulationofAT}), then $\x_k = \x + \epsilon \v_k$.
\end{lemma}

\begin{proof}
\textbf{First part.}
\begin{align}
& \lim_{\alpha\to\infty} \Pi_{\{|| \cdot ||_p = 1\}}(\v +\alpha \tilde\v) \\
&= \lim_{\alpha\to\infty}\argmin_{\v^* : ||\v^*||_p = 1} || \v^* - \v -\alpha \tilde\v||_2 \\
&= \lim_{\alpha\to\infty}\argmin_{\v^* : ||\v^*||_p = 1}  (\v^* - \v -\alpha \tilde\v)^\top (\v^* - \v -\alpha \tilde\v) \\
&= \lim_{\alpha\to\infty}\argmin_{\v^* : ||\v^*||_p = 1}  \v^{*\top} \v^* -2 \v^{*\top} \v - 2\alpha \v^{*\top} \tilde\v + \text{const}
\end{align}
where the $\text{const}$ term is independent of $\v^*$ and thus irrelevant for the $\argmin$.
Next, we observe that in the limit $\alpha\to\infty$, the $\v^{*\top} \v^* -2 \v^{*\top} \v$ term vanishes relative to the $\alpha \v^{*\top} \tilde\v$ term, hence
\begin{align}
& \lim_{\alpha\to\infty} \Pi_{\{|| \cdot ||_p = 1\}}(\v +\alpha \tilde\v) \\
&= \argmin_{\v^* : ||\v^*||_p = 1} - \v^{*\top} \tilde\v \\
&= \argmax_{\v^* : ||\v^*||_p = 1} \v^{*\top} \tilde\v \\
& = \frac{ \textnormal{sign}(\tilde\v) \odot |\tilde\v|^{p^*-1} }{ || \tilde\v ||_{p^*}^{p^*-1} }
\end{align}
where in the last line we have used Equation~\ref{eq:ATmaximizer} for the optimal perturbation to a linear function under an $\ell_p$-norm constraint that we have proven in Lemma~\ref{optimalpert_lemma}.

Moreover, if $\tilde\v$ is of the form $\tilde\v = \textnormal{sign}(\z) \odot |\z|^{p^*-1} / || \z ||_{p^*}^{p^*-1}$, then
\begin{align}
& \frac{\textnormal{sign}(\tilde\v) \odot |\tilde\v|^{p^*-1} }{ || \tilde\v ||_{p^*}^{p^*-1} }
= \frac{\textnormal{sign}(\z) \odot |\z|^{(p^*-1)(p^*-1)} }{ || \,| \z |^{p^*-1}\, ||_{p^*}^{p^*-1} } 
\end{align}
Now, observe that for $p \in \{1,2,\infty \}$, the H{\"o}lder conjugate $p^*\in\{1,2,\infty\}$. In particular, for these values $p^*$ satisfies $(p^*\!-\!1)(p^*\!-\!1) = p^*\!-\!1$ (since $0\cdot0=0, 1\cdot1=1, \infty\cdot\infty=\infty$).

Thus, the numerator (and hence the direction) remains the same.
Moreover, we also have that $|| \,| \z |^{p^*-1}\, ||_{p^*}^{p^*-1} = || \z ||_{p^*}^{p^*-1}$ (and hence the magnitude remains the same, too).
For $p^*\!=\!1$, $|| \,| \z |^{0}\, ||_{1}^{0} = || \mathbf{1} ||_1^0 = 1 = || \z ||_1^0 $ for any non-zero $\z$.
For $p^*\!=\!2$, $|| \,| \z |^{1}\, ||_{2}^{1} = || \z ||_2$ for any~$\z$.
For the $p^*\!=\!\infty$ case, we consider the full expression $ \textnormal{sign}(\z) \odot |\z|^{(p^*-1)(p^*-1)} / || \,| \z |^{p^*-1}\, ||_{p^*}^{p^*-1} =  | \mathcal{I}'|^{-1} \textnormal{sign}(\z) \odot \mathbf{1}_{\mathcal{I}'} $ where $\mathcal{I}' := \{ j \in [1,...,n] : |z_j|^{p^*-1} = || | \z |^{p^*-1} ||_\infty \}$ denotes the set of indices at which $| \z |^{p^*-1}$ attains its maximum norm. It is easy to see that $\mathcal{I}' = \mathcal{I}$, i.e.\ the set of indices of maximal elements remains the same.

Thus, if $\tilde\v$ is of the form $\tilde\v = \textnormal{sign}(\z) \odot |\z|^{p^*-1} / || \z ||_{p^*}^{p^*-1}$, then it is a fix point of Equation~\ref{eq:ATmaximizer} and hence of the projection 
\begin{align}
\lim_{\alpha\to\infty} \Pi_{\{|| \cdot ||_p = 1\}}(\v +\alpha \tilde\v) = \frac{\textnormal{sign}(\tilde\v) \odot |\tilde\v|^{p^*-1} }{ || \tilde\v ||_{p^*}^{p^*-1} } = \tilde\v \,\,\ \text{for} \,\,\ p^*\in\{1,2,\infty\}
\end{align}
Another way to see this is by checking that the operation $ \textnormal{sign}( \cdot ) \odot |\cdot|^{p^*-1} / || \cdot ||_{p^*}^{p^*-1}$, for $p^*\in\{1,2,\infty\}$, leaves the corresponding expressions for the optimal perturbation in the RHS of Equation~\ref{eq:optimalperturbationspecialcases} invariant.

Finally, note how the Projection Lemma implies that 
\begin{equation}
\lim_{\alpha \to \infty} \Pi_{\{|| \cdot ||_p = 1\}}(\alpha \z)\ = \argmax_{\v : || \v ||_p \leq 1} \v^\top \z \ ,
\end{equation}
for an arbitrary non-zero $\z$, which is an interesting result in its own right.

\newpage
\textbf{Second part.}

By definition of the orthogonal projection, we have
\begin{align}
\hspace{-2mm}\x_k & = \lim_{\alpha \to \infty} \Pi_{\mathcal{B}^p_{\epsilon}(\x) } (  \x_{k-1} + \alpha \v_k ) \\
&= \lim_{\alpha \to \infty} \argmin_{\x^* : || \x^* - \x ||_p \leq \epsilon} || \x^* \!-\! \x_{k-1} \!-\! \alpha \v_k ||_2 \\
&= \lim_{\alpha \to \infty} \argmin_{\x + \v^* : || \v^* ||_p \leq \epsilon} || \x \!+\! \v^* \!-\! \x_{k-1} \!-\! \alpha \v_k ||_2 \\
&= \x +  \lim_{\alpha \to \infty} \argmin_{\v^* : || \v^* ||_p \leq \epsilon} || \x \!+\! \v^* \!-\! \x_{k-1} \!-\! \alpha \v_k ||_2^2 \\
&= \x + \lim_{\alpha \to \infty} \argmin_{\v^* : || \v^* ||_p \leq \epsilon} \big\{ || \v^* \!-\! \alpha\v_k ||_2^2 + 2\v^{*\top}\text{const} \!+\! \text{const}^2 \big\} \\
&= \x + \lim_{\alpha \to \infty} \argmin_{\v^* : || \v^* ||_p \leq \epsilon} || \v^* - \alpha\v_k ||_2^2 \\
&= \x + \lim_{\alpha \to \infty} \argmin_{\v^* : || \v^* ||_p \leq \epsilon} || \v^* - \alpha\v_k ||_2 \\
&= \x + \lim_{\alpha \to \infty} \argmin_{\v^* : || \v^* ||_p = \epsilon} || \v^* - \alpha\v_k ||_2 \\
&= \x + \lim_{\alpha \to \infty} \Pi_{\{|| \cdot ||_p = \epsilon\}}(\alpha {\v}_k)
\end{align}
where (i) we have used that the $\text{const}^2$ term is independent of $\v^*$ and thus irrelevant for the $\argmin$, (ii) in the fourth-to-last line we have dropped all the terms that vanish relative to the limit $\alpha\to\infty$, and (iii) since $\alpha\v_k$ is outside the $\ell_p$-ball in the limit $\alpha\to\infty$, projecting into the $\ell_p$-ball $\{ \v^* : || \v^* ||_p \leq 1 \}$ is equivalent to projecting onto its boundary $\{ \v^* : || \v^* ||_p = 1 \}$.

By the first part of the Projection Lemma, we also have that 
\begin{align}
\x_k &= \x + \lim_{\alpha \to \infty} \Pi_{\{|| \cdot ||_p = \epsilon\}}(\alpha {\v}_k) \\
&= \x + \epsilon\, \frac{ \textnormal{sign}(\v_k) \odot |\v_k|^{p^*-1} }{ || \v_k ||_{p^*}^{p^*-1} } \\
&= \x + \epsilon \v_k \quad \text{if} \quad \v_k = \frac{\textnormal{sign}(\z) \odot |\z|^{p^*-1} }{ || \z ||_{p^*}^{p^*-1} }
\end{align}
for an arbitrary non-zero vector $\z$.
This completes the proof of the Lemma.

For illustrative purposes, we provide an alternative proof of the second part of the Lemma for $p \in \{ 2,\infty\}$, because for these values the projection can be written down explicitly away from the $\alpha\to\infty$ limit (the $p=1$ projection can be written down in a simple form in the $\alpha\to\infty$ limit only).

We first consider the \textbf{case} $p=2$.
The $\ell_2$-norm ball projection 
can be expressed as follows,
\begin{align}
\hspace{-3mm}& \Pi_{\mathcal{B}^2_{\epsilon}(\x) } (  \tilde{\x}_k )  = \x + \epsilon (\tilde{\x}_{k} - \x) / \max (\epsilon,  || \tilde{\x}_{k} - \x ||_2),
\end{align}
where $\tilde{\x}_{k} =  \x_{k-1} + \alpha \v_{k}$ and where the $\max (\epsilon,  || \tilde{\x}_{k} - \x ||_2)$ ensures that if $ || \tilde{\x}_{k} - \x ||_2 < \epsilon$ then $\x_{k} = \tilde{\x}_{k}$, i.e.\ we only need to project $\tilde{\x}_k$ if it is outside the $\epsilon$-ball $\mathcal{B}^2_{\epsilon}(\x)$.

Thus, in the limit $\alpha \to \infty$,
\begin{align}
& \lim_{\alpha \to \infty} \Pi_{\mathcal{B}^2_{\epsilon}(\x) } (  \x_{k-1} + \alpha \v_k ) \\
&= \lim_{\alpha \to \infty}  \x + \epsilon \frac{ \tilde{\x}_{k} - \x }{\max (\epsilon,  || \tilde{\x}_{k} - \x ||_2) } \\
& = \x + \epsilon \lim_{\alpha\to\infty} \frac{ \x_{k-1} + \alpha \v_{k} - \x}{\max (\epsilon,  || \x_{k-1} + \alpha \v_{k} - \x ||_2) } \\
& = \x + \epsilon \lim_{\alpha\to\infty} \frac{  \alpha ( \v_{k} + \frac{1}{ \alpha}(\x_{k-1} - \x) )}{\max (\epsilon,  \alpha || \v_{k} + \frac{1}{ \alpha}(\x_{k-1}  - \x) ||_2) } \\
& = \x + \epsilon \lim_{\alpha\to\infty} \frac{  \alpha ( \v_{k} + \frac{1}{ \alpha}(\x_{k-1} - \x) )}{ \alpha || \v_{k} + \frac{1}{ \alpha}(\x_{k-1}  - \x) ||_2 } \\[1mm]
&= \x + \epsilon\lim_{\alpha\to\infty}\Pi_{\{|| \cdot ||_2 = 1\}}(\alpha{\v}_k) \\[1mm]
& = \x + \epsilon \v_{k} / || \v_k ||_2 \label{eq:limitalphatoinfty2} \\[1mm]
&= \x + \epsilon \v_k \,\,\ \text{if} \,\ \v_k = \tilde\v_k / || \tilde\v_k ||_2
\end{align}
where in the fourth-to-last line we used that the $\max$ will be attained at its second argument in the limit $\alpha\to\infty$ since $|| \v_{k} + \frac{1}{ \alpha}(\x_{k-1}  - \x) ||_2 > 0$ for $\v_k \neq 0$,
and the last line holds if $\v_k$ is of the form $\tilde{\v}_k / || \tilde{\v}_k ||_2$. 

Next, we consider the \textbf{case} $p=\infty$.
The $\ell_\infty$-norm ball projection (clipping) can be expressed as follows,
\begin{align}
&  \Pi_{\mathcal{B}^{\infty}_{\epsilon}(\x) } ( \tilde{\x}_k )  = \x + \max( -\epsilon, \min( \epsilon, \tilde \x_k - \x )) \,,
\end{align}
where $\tilde{\x}_{k} =  \x_{k-1} + \alpha \v_{k}$ and where the $\max$ and $\min$ are taken \emph{elementwise}. Note that the order of the $\max$ and $\min$ operators can be exchanged, as we prove in the ``min-max-commutativity'' Lemma~\ref{max-min-lemma} below.

Thus, in the limit $\alpha \to \infty$,
\begin{align}
\hspace{-1mm}& \lim_{\alpha \to \infty} \Pi_{\mathcal{B}^\infty_{\epsilon}(\x) } (  \x_{k-1} + \alpha \v_k ) - \x \\
\hspace{-1mm}&=  \lim_{\alpha \to \infty} \max( -\epsilon, \min( \epsilon, \alpha \v_k + \x_{k-1} - \x)) \\
\hspace{-1mm}&= \hspace{-1mm}\lim_{\alpha \to \infty} \hspace{-1mm}\Big\{ \mathds{1}_{\{ \textnormal{sign}(\v_k) > 0 \}} \max( -\epsilon, \min( \epsilon, \alpha |\v_k| + \x_{k-1} \!-\! \x)) \nonumber\\ 
&+ \mathds{1}_{\{ \textnormal{sign}(\v_k) < 0 \}} \max( -\epsilon, \min( \epsilon, -\alpha |\v_k| + \x_{k-1} \!-\! \x)) \Big\} 
\end{align}
where in going from the second to the third line we used that $\v_k = \textnormal{sign}(\v_k) \odot |\v_k|$.

Next, observe that 
\begin{align}
\lim_{\alpha \to \infty} \max( -\epsilon, \min( \epsilon, \alpha |\v_k| + \x_{k-1} - \x)) \\
= \lim_{\alpha \to \infty} \min( \epsilon, \alpha |\v_k| + \x_{k-1} - \x) = \epsilon
\end{align}
since $\min( \epsilon, \alpha |\v_k| + \x_{k-1} - \x) > -\epsilon$.

Similarly, we have
\begin{align}
& \lim_{\alpha \to \infty} \max( -\epsilon, \min( \epsilon, -\alpha |\v_k| + \x_{k-1} - \x)) \\
&= \lim_{\alpha \to \infty} \min( \epsilon, \max( -\epsilon, -\alpha |\v_k| + \x_{k-1} - \x)) \\
&= \lim_{\alpha \to \infty}  \max( -\epsilon, -\alpha |\v_k| + \x_{k-1} - \x) \\
&= -\epsilon
\end{align}
where for the first equality we have used the ``min-max-commutativity'' Lemma~\ref{max-min-lemma} below, which asserts that the order of the $\max$ and $\min$ can be exchanged, while the second equality holds since $\max( -\epsilon, -\alpha |\v_k| + \x_{k-1} - \x) < \epsilon$.

With that, we can continue 
\begin{align}
& \lim_{\alpha \to \infty} \Pi_{\mathcal{B}^\infty_{\epsilon}(\x) } (  \x_{k-1} + \alpha \v_k ) \\
&= \x +  \epsilon \mathds{1}_{\{ \textnormal{sign}(\v_k) > 0 \}} - \epsilon\mathds{1}_{\{ \textnormal{sign}(\v_k) < 0 \}} \\[1mm]
&= \x + \epsilon \,\textnormal{sign}(\v_k) \\[1mm]
&= \x + \epsilon \lim_{\alpha \to \infty} \Pi_{\{|| \cdot ||_\infty = 1\}}(\alpha{\v}_k) \\
&= \x + \epsilon \, \v_k \,\,\ \text{if} \,\ \v_k = \textnormal{sign}(\tilde{\v}_k)
\end{align}
where the last line holds if $\v_k$ is of the form $\textnormal{sign}(\z)$, since $\textnormal{sign}(\v_k) = \textnormal{sign}( \textnormal{sign}(\tilde{\v}_k) ) = \textnormal{sign}(\tilde{\v}_k) = \v_k$.
Note that the last line can directly be obtained from the third-to-last line if $\v_k = \textnormal{sign}(\tilde{\v}_k)$.

\bigskip
Finally, for the \textbf{general case}, we provide the following additional intuition. 
Taking the limit ${\alpha\to\infty}$ has two effects: firstly, it puts all the weight in the sum $\x_{k-1} + \alpha \v_k$ on $\alpha\v_k$ and secondly, it takes every component of $\v_k$ out of the $\epsilon$-ball $\mathcal{B}^p_{\epsilon}(\x)$. As a result, $\Pi_{\mathcal{B}^p_{\epsilon}(\x) }$ will project $\alpha \v_k$ onto a point on the boundary of the $\epsilon$-ball, which is precisely the set $\{\v : || \v  ||_p = \epsilon\}$. Hence, $\lim_{\alpha \to \infty} \Pi_{\mathcal{B}^p_{\epsilon}(\x) } (  \x_{k-1} + \alpha \v_k ) = \x + \lim_{\alpha \to \infty} \Pi_{\{|| \cdot ||_p = \epsilon\}}(\alpha {\v}_k) $. 
\end{proof}

\bigskip
Finally, we provide another very intuitive yet surprisingly hard\footnote{The proof is easier yet less elegant if one alternatively resorts to case distinctions.} to prove result:
\begin{lemma}[Min-max-commutativity]\label{max-min-lemma}
The order of the \emph{elementwise} $\max$ and $\min$ in the projection (clipping) operator $\Pi_{\mathcal{B}^{\infty}_{\epsilon}(\x)} $ can be exchanged, i.e.\ 
\begin{align}
\max( -\epsilon, \min( \epsilon, \x)) = \min( \epsilon, \max( -\epsilon, \x))
\end{align}

\end{lemma}
\begin{proof}

We are using the following representations which hold \textit{elementwise} for all $\mathbf{a}, \x \in \Re^n$:
\begin{align}
\hspace{-4mm}\max(\mathbf{a}, \x) &= \mathbf{a} + \max(0, \x \!-\! \mathbf{a}) = \mathbf{a} + \mathds{1}_{\{ \mathbf{a} < \x \}}( \x \!-\! \mathbf{a}) \\
\hspace{-4mm}\min(\mathbf{a}, \x) &= \mathbf{a} + \min(0, \x \!-\! \mathbf{a}) = \mathbf{a} + \mathds{1}_{\{ \x < \mathbf{a} \}}( \x \!-\! \mathbf{a}) 
\end{align}
where $\mathds{1}_{\{ \cdot \}}$ denotes the elementwise indicator function.

With these, we have
\begin{align}
& \max( -\epsilon, \min( \epsilon, \x)) \\
&= -\epsilon +  \mathds{1}_{\{ -\epsilon < \min(\epsilon, \x) \}}( \min(\epsilon, \x) + \epsilon) \\
&= -\epsilon +  \mathds{1}_{\{ -\epsilon < \x \}}( 2\epsilon + \mathds{1}_{\{ \x < \epsilon \}} (\x - \epsilon)) \\
&= -\epsilon +  2\epsilon \mathds{1}_{\{ -\epsilon < \x \}} + \mathds{1}_{\{ -\epsilon < \x < \epsilon \}} (\x - \epsilon) \\
&= -\epsilon +  2\epsilon (\mathds{1}_{\{ -\epsilon < \x < \epsilon \}} + \mathds{1}_{\{ \epsilon < \x \}} ) + \mathds{1}_{\{ -\epsilon < \x < \epsilon \}} (\x - \epsilon) \\
&= -\epsilon +  2\epsilon \mathds{1}_{\{ \epsilon < \x \}} + \mathds{1}_{\{ -\epsilon < \x < \epsilon \}} (\x + \epsilon) \\
&= -\epsilon +  2\epsilon ( 1 - \mathds{1}_{\{ \x < \epsilon \}}) + \mathds{1}_{\{ -\epsilon < \x < \epsilon \}} (\x + \epsilon) \\
&= \epsilon + \mathds{1}_{\{ \x < \epsilon \}} (-2\epsilon + \mathds{1}_{\{ -\epsilon < \x \}}(\x + \epsilon)) \\
&= \epsilon + \mathds{1}_{\{ \max(-\epsilon, \x) < \epsilon \}} ( \max(-\epsilon, \x) - \epsilon ) \\
&= \min( \epsilon, \max( -\epsilon, \x))
\end{align}
where we used that $\mathds{1}_{\{ -\epsilon < \min(\epsilon, \x) \}} = \mathds{1}_{\{ -\epsilon < \x \}}$ and 
$\mathds{1}_{\{ \max(-\epsilon, \x) < \epsilon \}} = \mathds{1}_{\{ \x < \epsilon \}}$.
\end{proof}

\newpage
\subsection{Proof of Main Theorem}
\label{sec:proof}

The conditions on $\epsilon$ and $\alpha$ can be considered specifics of the respective iteration method. 
The condition that $\epsilon$ be small enough such that $\mathcal{B}^p_{\epsilon}(\x)$ is contained in the ReLU cell around $\x$ ensures that  $\Jac{\x^*} = \Jac{\x}$ for all $\x^* \in \mathcal{B}^p_{\epsilon}(\x)$. 
The power-method limit $\alpha \to \infty$ means that in the update equations all the weight (no weight) is placed on the current gradient direction (previous iterates).

To proof the theorem we need to show that the updates for $\ell_p$-norm constrained projected gradient ascent based adversarial training with an $\ell_q$-norm loss on the logits reduce to the corresponding updates for data-dependent operator norm regularization in Equation~\ref{eq:datadepoperatornorm_powermethodlimit} under the above conditions on $\epsilon$ and $\alpha$.

\begin{proof}
For an $\ell_q$-norm loss on the logits of the clean and perturbed input $\ell_{\rm adv}( f(\x), f(\x^*)) = || f(\x) - f(\x^*) ||_q$, the corresponding $\ell_p$-norm constrained projected gradient ascent updates in Equation~\ref{eq:powermethodformulationofAT} are
\begin{equation}
\begin{aligned}
\label{eq:powermethodformulationofAT_ellq}
\u_k & \leftarrow \frac{ \textnormal{sign}(\tilde\u_k) \odot |\tilde\u_k|^{q-1} }{ ||\tilde\u_k||_q^{q-1} } \,\ , \,\ \tilde\u_k \leftarrow f(\x_{k-1})\! -\! f(\x)  \\
\v_k & \leftarrow \frac{ \textnormal{sign}(\tilde\v_k) \odot |\tilde\v_k|^{p^*-1} }{ || \tilde\v_k||_{p^*}^{p^*-1} } \,\ ,  \,\,\ \tilde{\v}_k \leftarrow  \left. \Jac{\x_{k-1}}^\top \u_k \right._{} \\
\x_{k} &\leftarrow \Pi_{\mathcal{B}^p_{\epsilon}(\x) } (  \x_{k-1} + \alpha \v_k ) 
\end{aligned}
\end{equation}
In the limit $\alpha\to\infty$, $\x_{k-1} = \x + \epsilon \v_{k-1}$ by the ``Projection Lemma'' (Lemma~\ref{projection_lemma}) and thus for small enough $\epsilon$, $f(\x_{k-1}) - f(\x) = \Jac{\x}(\x_{k-1}-\x) = \epsilon\Jac{\x}\v_{k-1}$ (equality holds because $\x_{k-1} \in \mathcal{B}^p_{\epsilon}(\x) \subset X(\phi_\x)$). 
Thus, the forward pass becomes
\begin{align}
\u_k & \leftarrow \frac{ \textnormal{sign}(\tilde\u_k) \odot |\tilde\u_k|^{q-1} }{ ||\tilde\u_k||_q^{q-1} } \,\ , \,\ \tilde\u_k \leftarrow \Jac{\x} \v_{k-1}     
\end{align}
For the backward pass, we have
\begin{align}
\tilde{\v}_k = \left. \Jac{\x_{k-1}}^\top \u_k \right._{} =  \Jac{\x}^\top \u_k  \, ,
\end{align}
since $\Jac{\x_k} = \Jac{\x}$ for all $\x_k \in \mathcal{B}^p_{\epsilon}(\x) \subset X(\phi_\x)$.
Note that the update equation for $\x_k$ is not needed since the Jacobians in the forward and backward passes don't depend on $\x_k$ for $\mathcal{B}^p_{\epsilon}(\x) \subset X(\phi_\x)$. 
The update equations for $\ell_p$-norm constrained projected gradient ascent based adversarial training with an $\ell_q$-norm loss on the logits can therefore be written as
\begin{equation}
\begin{aligned}
\u_k & \leftarrow \textnormal{sign}(\tilde\u_k) \odot |\tilde\u_k|^{q-1} / ||\tilde\u_k||_q^{q-1} \,\ , \,\ \tilde\u_k \leftarrow \Jac{\x} \v_{k-1}  \\
\v_k & \leftarrow \textnormal{sign}(\tilde\v_k) \odot |\tilde\v_k|^{p^*-1} / || \tilde\v_k||_{p^*}^{p^*-1} ,  \ \tilde{\v}_k \leftarrow  \left. \Jac{\x}^\top \u_k \right._{}   
\end{aligned}
\end{equation}
which is precisely the power method limit of (p, q) operator norm regularization in Equation~\ref{eq:datadepoperatornorm_powermethodlimit}.
We have thus shown that the update equations to compute the adversarial perturbation and the data-dependent operator norm maximizer are exactly the same.

It is also easy to see that the objective functions used to update the network parameters 
for $\ell_p$-norm constrained projected gradient ascent based adversarial training with an $\ell_q$-norm loss on the logits of clean and adversarial inputs in Equation~\ref{eq:ATobjective} 
\begin{align}
& \E_{(\x,y )\sim \hat{P}}\Big[ \ell(y, f(\x)) + \lambda \max_{\x^* \in \mathcal{B}^p_{\epsilon}(\x)} || f(\x) - f(\x^*) ||_q \Big] 
\end{align}
is by the condition $\mathcal{B}^p_{\epsilon}(\x) \subset X(\phi_\x)$ and $\x^* = \x + \epsilon \v$
\begin{align}
\E_{(\x,y )\sim \hat{P}}\Big[ \ell(y, f(\x)) +   \lambda\epsilon \max_{\v^* : || \v^* ||_p \leq 1} || \J_{f(\x)} \v ||_q \Big] 
\end{align}
the same as that of data-dependent (p, q) operator norm regularization in Equation~\ref{eq:operatornormregularization}.

\end{proof}

\newpage
We conclude this section with a note on generalizing our Theorem to allow for activation pattern changes.
Proving such an extension for the approximate correspondence between adversarial training and data-dependent operator norm regularization that we observe in our experiments is highly non-trivial, 
as this requires to take into account how much ``nearby'' Jacobians can change based on the crossings of ReLU boundaries,
which is complicated by the fact that the impact of such crossings depends heavily on the specific activation pattern at input $\mathbf{x}$ and the precise values of the weights and biases in the network.
We consider this to be an interesting avenue for future investigations.

\subsection{Extracting Jacobian as a Matrix}
\label{sec:extractingJacobian}

Since we know that any neural network with its nonlinear activation function set to fixed values represents a linear operator, which, locally, is a good approximation to the neural network itself, we develop a method to fully extract and specify this linear operator in the neighborhood of any input datapoint $\x$. We have found the naive way of determining each entry of the linear operator by consecutively computing changes to individual basis vectors to be numerically unstable and therefore have settled for a more robust alternative:

In a first step, we run a set of randomly perturbed versions of $\x$ through the network (with fixed activation functions) and record their outputs at the particular layer that is of interest to us (usually the logit layer).
In a second step, we compute a linear regression on these input-output pairs to obtain a weight matrix $\W$ as well as a bias vector $\b$, thereby fully specifying the linear operator.
The singular vectors and values of $\W$ can be obtained by performing an SVD.

\subsection{Dataset, Architecture \& Training Methods} 
\label{sec:experimentalsetup}

We trained Convolutional Neural Networks (CNNs) with seven hidden layers and batch normalization on the CIFAR10 data set \cite{krizhevsky2009learning}.
The CIFAR10 dataset consists of $60$k $32\times32$ colour images in $10$ classes, with $6$k images per class. 
It comes in a pre-packaged train-test split, with $50$k training images and $10$k test images,
and can readily be downloaded from \url{https://www.cs.toronto.edu/~kriz/cifar.html}.

We conduct our experiments on a pre-trained standard convolutional neural network, employing 7 convolutional layers, augmented with BatchNorm, ReLU nonlinearities and MaxPooling.
The network achieves 93.5\% accuracy on a clean test set. Relevant links to download the pre-trained model can be found in our codebase.
For the robustness experiments, we also train a state-of-the-art Wide Residual Net (WRN-28-10) \cite{zagoruyko2016wide}. The network achieves 96.3\% accuracy on a clean test set.

We adopt the following standard preprocessing and data augmentation scheme:
Each training image is zero-padded with four pixels on each side, randomly cropped to produce a new image with the original dimensions and horizontally flipped with probability one half. 
We also standardize each image to have zero mean and unit variance when passing it to the classifier.

We train each classifier with a number of different training methods: 
(i) `Standard':\ standard empirical risk minimization with a softmax cross-entropy loss, 
(ii) `Adversarial':\ $\ell_2$-norm constrained projected gradient ascent (PGA) based adversarial training with a softmax cross-entropy loss,
(iii) `global SNR':\ global spectral norm regularization {\`a} la Yoshida \& Miyato~\cite{yoshida2017spectral}, 
and (iv) `d.-d.\ SNR':\ data-dependent spectral norm regularization. 
For the robustness experiments, we also train a state-of-the-art Wide Residual Network (WRN-28-10) \cite{zagoruyko2016wide}.

As a default attack strategy we use an $\ell_2$- \& $\ell_\infty$-norm constrained PGA white-box attack with cross-entropy adversarial loss $\ell_{\rm adv}$ and 10 attack iterations. 
We verified that all our conclusions also hold for larger numbers of attack iterations, however, due to computational constraints we limit the attack iterations to 10.
The attack strength~$\epsilon$ used for PGA was chosen to be the smallest value such that almost all adversarially perturbed inputs to the standard model are successfully misclassified, which is $\epsilon=1.75$ (for $\ell_2$-norm) and $\epsilon=8/255$ (for $\ell_\infty$-norm). 

The regularization constants of the other training methods were then chosen in such a way that they roughly achieve the same test set accuracy on clean examples as the adversarially trained model does, i.e.\ we allow a comparable drop in clean accuracy for regularized and adversarially trained models.
When training the derived regularized models, we started from a pre-trained checkpoint and ran a hyper-parameter search over number of epochs, learning rate and regularization constants.

Table~\ref{tbl:hypers} summarizes the test set accuracies and hyper-parameters for all the training methods we considered.

\begin{table}[t]
\centering
	\caption{CIFAR10 test set accuracies and hyper-parameters for the models and training methods we considered. The regularization constants were chosen such that the models achieve roughly the same accuracy on clean test examples as the adversarially trained model does. See  Table~\ref{tbl:hypersweep} for the full hyper-parameter sweep.}\label{tbl:hypers}
	\vskip 0.05in
	\tiny
	\begin{sc}
	\begin{tabular}{lll}
		\toprule
		Model \& Training Method & Acc & Hyper-parameters \\
		\midrule
		\emph{CNN7} \\[1mm]
		Standard Training & 93.5\% & --- \\[1mm]
		Adversarial Training ($\ell_2\!-\!norm$) & 83.6\% & $\epsilon = 1.75 , \alpha = 2\epsilon / \text{iters},$ iters $= 10$ \\[1mm]
		Adversarial Training ($\ell_\infty\!-\!norm$) & 82.9\% & $\epsilon = 8/255 , \alpha = 2\epsilon / \text{iters},$ iters $= 10$ \\[1mm]
		Data-dep.\ Spectral Norm\ & 84.6\% & $\lambda = 3 \cdot 10^{-2} , \epsilon = 1.75,$ iters $= 10$ \\[1mm]
		Data-dep.\ Operator Norm\ & 83.0\% & $\lambda = 3 \cdot 10^{-2} , \epsilon = 8/255,$ iters $= 10$ \\[1mm]
		Global Spectral Norm\ & 81.5\% & $\lambda = 3 \cdot 10^{-4},$ iters = $1, 10$ \\[2mm]
		
		\emph{WRN-28-10} \\[1mm]
		Standard Training & 96.3\% & --- \\[1mm]
		Adversarial Training ($\ell_2\!-\!norm$) & 91.8\% & $\epsilon = 1.75 , \alpha = 2\epsilon / \text{iters},$ iters $= 10$ \\[1mm]
		Data-dep.\ Spectral Norm\ & 91.3\% & $\lambda = 3 \cdot 10^{-1} , \epsilon = 1.75,$ iters $= 10$ \\[1mm]
		\bottomrule
	\end{tabular}
\end{sc}
\vspace{4mm}
\end{table}

\medskip
\subsection{Hyperparameter Sweep}
\label{sec:hyperparamsweep}

\begin{table*}[h!]
\centering
	\caption{Hyperparameter sweep during training. We report results for the best performing models in the main text.}\label{tbl:hypersweep}
	\vskip 0.05in
	\scriptsize
	\begin{sc}
	\begin{tabular}{lll}
		\toprule
		Training Method & Hyperparameter & Values Tested \\
		\midrule
        Adversarial Training & $\epsilon$ ($\ell_2$\!-\!norm) & $0.5, 0.75, 1.0, 1.25, 1.5, 1.75, 2.0,$\\
        & & $2.5, 2.75, 3.0, 3.25, 3.5, 3.75, 4.0$ \\
        & $\epsilon$ ($\ell_\infty$\!-\!norm) & $1/255, 2/255, 3/255, 4/255, 5/255, 6/255, 7/255, 8/255,$\\
        & & $9/255, 10/255, 11/255, 12/255, 13/255, 14/255, 15/255,$ \\
        & & $16/255, 17/255, 18/255$ \\
        & $\alpha$ & $\epsilon/\text{iters}, 2\epsilon/\text{iters}, 3\epsilon/\text{iters}, 4\epsilon/\text{iters}, 5\epsilon/\text{iters}$ \\
        & $\text{iters}$ & $1, 2, 3, 5, 8, 10, 15, 20, 30, 40, 50$ \\
        Global Spectral Norm Reg.\ & $\lambda$ & $1 \cdot 10^{-5}, 3 \cdot 10^{-5}, 1 \cdot 10^{-4}, 3 \cdot 10^{-4},$\\
        & & $1 \cdot 10^{-3}, 3 \cdot 10^{-3}, 1 \cdot 10^{-2}, 3 \cdot 10^{-2}, 1 \cdot 10^{-1}, 3 \cdot 10^{-1},$\\
        & & $1 \cdot 10^{0}, 3 \cdot 10^{0}, 1 \cdot 10^{1}, 3 \cdot 10^{1} $ \\
        & $\text{iters}$ & $1, 10$\\
		Data-dep.\ Spectral Norm Reg.\ & $\lambda$ & $1 \cdot 10^{-5}, 3 \cdot 10^{-5}, 1 \cdot 10^{-4}, 3 \cdot 10^{-4},$\\
        & & $1 \cdot 10^{-3}, 3 \cdot 10^{-3}, 1 \cdot 10^{-2}, 3 \cdot 10^{-2}, 1 \cdot 10^{-1}, 3 \cdot 10^{-1},$\\
        & & $1 \cdot 10^{0}, 3 \cdot 10^{0}, 1 \cdot 10^{1}, 3 \cdot 10^{1} $ \\
        & $\text{iters}$ & $1, 2, 3, 5, 8, 10, 15, 20, 30, 40, 50$ \\
		\bottomrule
	\end{tabular}
\end{sc}
\end{table*}

\newpage
\bigskip\bigskip\medskip
\textbf{\large Further Experimental Results}
\medskip

\subsection{Adversarial Training with Large $\alpha$}
\label{sec:furtherresultslargealpha}

Figure~\ref{fig:largealpha} shows the result of varying $\alpha$ in adversarial training.
As can be seen, the adversarial robustness initially rises with increasing $\alpha$, but after some threshold it levels out and does not change significantly even at very large values.

\begin{figure}[h!]
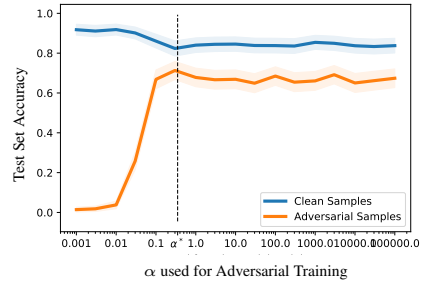

\centering
\begin{tabular}{@{}l@{\hspace{3pt}}c@{}}
\begin{turn}{90}\hspace{25pt} {\tiny Test Set Accuracy} \end{turn} & \adjincludegraphics[width=0.4\linewidth, trim={{0.07\width} {0.05\height} {0.0\width} {0.0\height}},clip]{plots/alpha.pdf} \\[-1mm]
 & {\tiny $\alpha$ used for Adversarial Training}
\end{tabular}
\vspace{-1mm}
    \caption{
    Test set accuracy on clean and adversarial examples for models adversarially trained with different PGA step sizes $\alpha$. 
    The dashed line indicates the $\alpha$ used when generating adversarial examples at test time.
    The $\epsilon$ in the AT projection was fixed to the value used in the main text.}
    \label{fig:largealpha}
\end{figure}

\subsection{Interpolating between AT and d.d.\ SNR} 

\begin{figure}[h!]
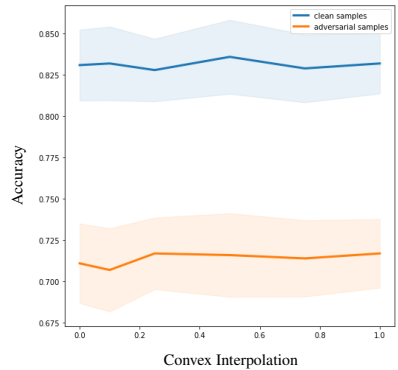

\centering
\begin{tabular}{@{}l@{\hspace{3pt}}c@{}}
\begin{turn}{90}\hspace{50pt} {\tiny Accuracy} \end{turn} & \adjincludegraphics[width=0.35\linewidth, trim={{0.0\width} {0.0\height} {0.0\width} {0.0\height}},clip]{plots5/interpolatingATddSNR.png} \\[-1mm]
 & \quad{\tiny Convex Interpolation}
\end{tabular}
\vspace{-1mm}
\caption{
Test set accuracy on clean and adversarial examples for different networks from scratch each with an objective function that convexly combines adversarial training with data-dependent spectral norm regularization in a way that allows us to interpolate between (i) the fraction of adversarial examples relative to clean examples used during adversarial training controlled by $\lambda$ in Eq.~\ref{eq:ATobjective} and (ii) the regularization parameter $\tilde\lambda$ in Eq.~\ref{eq:operatornormregularization}. The plot confirms that we can continuously trade-off the contribution of AT with that of d.d. SNR in the empirical risk minimization. 
}
\label{fig:interpolatingATddSNR}
\end{figure}

\subsection{Alignment of Adversarial Perturbations with Dominant Singular Vector} 

\label{sec:furtherresultsalignment}
Figure~\ref{fig:furtheralignment} shows the cosine-similarity of adversarial perturbations of mangitude $\epsilon$ with the dominant singular vector of $\Jac{\x}$, as a function of perturbation magnitude $\epsilon$. 
For comparison, we also include the alignment with random perturbations.
For all training methods, the larger the perturbation magnitude $\epsilon$, the lesser the adversarial perturbation aligns with the dominant singular vector of $\Jac{\x}$,
which is to be expected for a simultaneously increasing deviation from linearity. 
The alignment is similar for adversarially trained and data-dependent spectral norm regularized models and for both larger than that of global spectral norm regularized and naturally trained models.

\begin{figure}[h!]
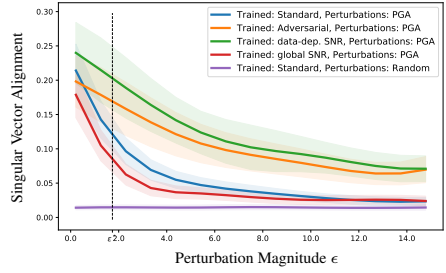

\centering
\begin{tabular}{@{}l@{\hspace{3pt}}c@{}}
\begin{turn}{90}\hspace{12pt} {\tiny Singular Vector Alignment} \end{turn} & \adjincludegraphics[width=0.4\columnwidth, trim={{0.037\width} {0.085\height} {0.0\width} {0.0\height}},clip]{plots/top_singular_vector} \\[-1mm]
 & \quad{\tiny Perturbation Magnitude $\epsilon$}
\end{tabular}
\vspace{-1mm}
    \caption{Alignment of adversarial perturbations with dominant singular vector of $\Jac{\x}$ as a function of perturbation magnitude $\epsilon$.
The dashed vertical line indicates the $\epsilon$ used during adversarial training. Curves were aggregated over $2000$ test samples.}
    \label{fig:furtheralignment}
\end{figure}

\newpage
\subsection{Activation Patterns} 

\begin{figure}[h!]
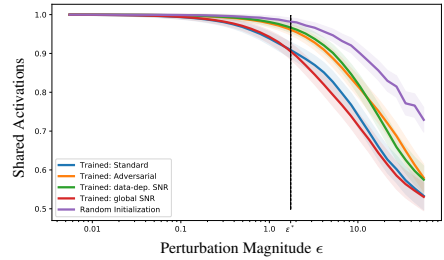

\centering
\begin{tabular}{@{}l@{\hspace{3pt}}c@{}}
\begin{turn}{90}\hspace{18pt} {\tiny Shared Activations} \end{turn} & \adjincludegraphics[width=0.4\linewidth, trim={{0.037\width} {0.085\height} {0.0\width} {0.068\height}},clip]{plots5/relu_masks_pgd.pdf} \\[-1mm]
 & {\tiny Perturbation Magnitude $\epsilon$}
\end{tabular}
\vspace{-1mm}
\caption{
Fraction of shared activations as a function of perturbation magnitude $\epsilon$ 
between activation patterns $\phi_\x \text{ and } \phi_{\x^*}$, 
where $\x$ is a data point sampled from the test set,
and $\x^*$ is an adversarially perturbed input, with perturbation magnitude $\epsilon$.}
\label{fig:furtherresultsactivationpatterns}
\end{figure}

\subsection{Global SNR with 10 iterations}
\label{sec:furtherresultsglobalsnr}

In the main section, we have implemented the baseline version of global SNR as close as possible to the descriptions in \cite{yoshida2017spectral}.
However, this included a recommendation from the authors to perform only a single update iteration to the spectral decompositions of the weight matrices per training step.
As this is computationally less demanding than the 10 iterations per training step spent on both adversarial training, as well as data-dependent spectral norm regularization, we verify that performing 10 iterations makes no difference to the method of \cite{yoshida2017spectral}.
Figures~\ref{fig:Linearity_globalsnr} and~\ref{fig:Accuracy_globalsnr} reproduce the curves for global SNR from the main part (having used 1 iteration) and overlap it with the same experiments, but done with global SNR using 10 iterations.
As can be seen, there is no significant difference.

\begin{figure}[h!]
\centering
\begin{tabular}{@{}l@{\hspace{1pt}}c@{\hspace{15pt}}l@{\hspace{1pt}}c@{}}
\begin{turn}{90}\hspace{20pt}  {\tiny Deviation from Linearity} \end{turn} & \adjincludegraphics[width=0.4\linewidth, trim={{0.042\width} {0.06\height} {0.0\width} {0.0\height}},clip]{plots5/linear} &
\begin{turn}{90}\hspace{25pt} {\tiny Top Singular Value} \end{turn} & \adjincludegraphics[width=0.4\linewidth, trim={{0.042\width} {0.06\height} {0.0\width} {0.0\height}},clip]{plots5/top_singular_value} \\[-1mm]
& \quad{\tiny Distance from $\x$} & & \quad{\tiny Distance from $\x$} 
\end{tabular}
\vspace{-1mm}
\caption{(Left) Deviation from linearity $|| \phi^{L-1}(\x + \z) - ( \phi^{L-1}(\x) + \J_{\phi^{L-1}}(\x) \z ) ||_2$ as a function of the distance $||\z||_2$ from $\x$ for random and adversarial perturbations $\z$. 
(Right) Largest singular value of the linear operator $\J_{f}(\x+\z)$ as a function of the magnitude $||\z||_2$ of random and adversarial perturbations $\z$. 
The dashed vertical line indicates the $\epsilon$ used during adversarial training. Curves were aggregated over $200$ samples from the test set.} 
\label{fig:Linearity_globalsnr}
\end{figure}

\begin{figure}[h!]
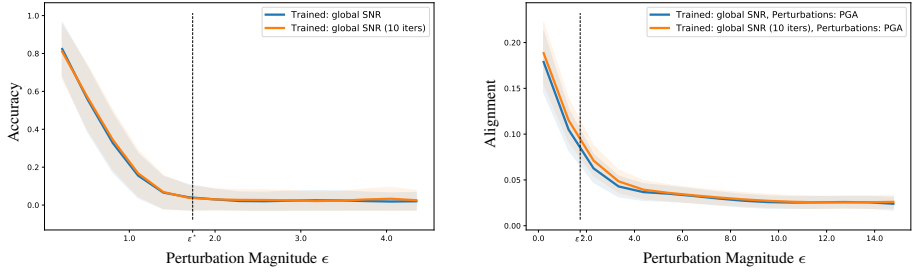

\centering
\begin{tabular}{@{}l@{\hspace{1pt}}c@{\hspace{15pt}}l@{\hspace{1pt}}c@{}}
\begin{turn}{90}\hspace{35pt} {\tiny Accuracy} \end{turn} & \adjincludegraphics[width=0.4\linewidth, trim={{0.037\width} {0.085\height} {0.0\width} {0.0\height}},clip]{plots5/adversarial_accuracy} &
\begin{turn}{90}\hspace{33pt} {\tiny Alignment} \end{turn} & \adjincludegraphics[width=0.4\linewidth, trim={{0.037\width} {0.085\height} {0.0\width} {0.0\height}},clip]{plots5/top_singular_vector} \\[-1mm]
& \quad{\tiny Perturbation Magnitude $\epsilon$} & & \quad{\tiny Perturbation Magnitude $\epsilon$}
\end{tabular}
\vspace{-1mm}
\caption{
(Left) Classification accuracy as a function of perturbation strength $\epsilon$. 
(Right) Alignment of adversarial perturbations with dominant singular vector of $\Jac{\x}$ as a function of perturbation magnitude $\epsilon$.
The dashed vertical line indicates the $\epsilon$ used during adversarial training. Curves were aggregated over $2000$ samples from the test set.}
\label{fig:Accuracy_globalsnr}
\end{figure}

\subsection{$\ell_\infty$-norm Constrained Projected Gradient Ascent}
\label{sec:furtherresults}

Additional results against $\ell_\infty$-norm constrained PGA attacks are provided in Figures~\ref{fig:Linearity_inf}~\&~\ref{fig:Accuracy_inf}.
Note that all adversarial and regularized training methods are robustifying against $\ell_2$ PGA, or regularizing the spectral (2, 2)-operator norm, respectively.
Results of adversarial training using $\ell_\infty$-norm constrained PGA and their equivalent regularization methods can be found in Section~\ref{sec:furtherresultsinf}.
The conclusions remain the same for all the experiments we conducted.
\medskip

\begin{figure}[h!]
\centering
\begin{tabular}{@{}l@{\hspace{1pt}}c@{\hspace{15pt}}l@{\hspace{1pt}}c@{}}
\begin{turn}{90}\hspace{20pt}  {\tiny Deviation from Linearity} \end{turn} & \adjincludegraphics[width=0.4\linewidth, trim={{0.042\width} {0.06\height} {0.0\width} {0.0\height}},clip]{plots/linear_inf} &
\begin{turn}{90}\hspace{23pt} {\tiny Top Singular Value} \end{turn} & \adjincludegraphics[width=0.4\linewidth, trim={{0.042\width} {0.06\height} {0.0\width} {0.0\height}},clip]{plots/top_singular_value_inf} \\[-1mm]
& \quad{\tiny Distance from $\x$} & & \quad{\tiny Distance from $\x$} 
\end{tabular}
\vspace{-1mm}
\caption{(Left) Deviation from linearity $|| \phi^{L-1}(\x + \z) - ( \phi^{L-1}(\x) + \J_{\phi^{L-1}}(\x) \z ) ||_2$ as a function of the distance $||\z||_2$ from $\x$ for random and $\ell_\infty$-PGA adversarial perturbations $\z$. 
(Right) Largest singular value of $\J_{\phi^{L-1}}(\x+\z)$ as a function of the magnitude $||\z||_2$ of random and $\ell_\infty$-PGA adversarial perturbations $\z$. 
The dashed vertical line indicates the $\epsilon$ used during adversarial training. Curves were aggregated over $200$ test samples.} 
\label{fig:Linearity_inf}
\end{figure}

\begin{figure}[h!]
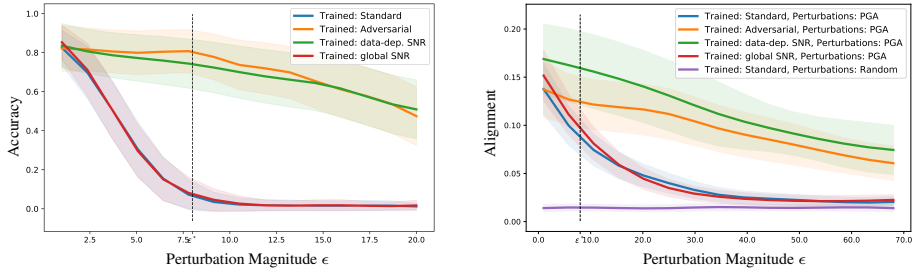

\centering
\begin{tabular}{@{}l@{\hspace{1pt}}c@{\hspace{15pt}}l@{\hspace{1pt}}c@{}}
\begin{turn}{90}\hspace{30pt} {\tiny Accuracy} \end{turn} & \adjincludegraphics[width=0.4\linewidth, trim={{0.037\width} {0.085\height} {0.0\width} {0.0\height}},clip]{plots/adversarial_accuracy_inf.png} &
\begin{turn}{90}\hspace{30pt} {\tiny Alignment} \end{turn} & \adjincludegraphics[width=0.4\linewidth, trim={{0.037\width} {0.085\height} {0.0\width} {0.0\height}},clip]{plots/top_singular_vector_inf} \\[-1mm]
& \quad{\tiny Perturbation Magnitude $\epsilon$} & & \quad{\tiny Perturbation Magnitude $\epsilon$}
\end{tabular}
\vspace{-1mm}
\caption{
(Left) Classification accuracy for an $\ell_2$-norm trained network on $\ell_\infty$-norm perturbations with $\epsilon$ (measured in 8-bit). 
(Right) Alignment of $\ell_\infty$-PGA adversarial perturbations with dominant singular vector of $\Jac{\x}$ as a function of perturbation magnitude $\epsilon$.
The dashed vertical line indicates the $\epsilon$ used during adversarial training. Curves were aggregated over $2000$ samples from the test set.}
\label{fig:Accuracy_inf}
\end{figure}

\clearpage

\subsection{Data-Dependent $\ell_\infty$-norm regularization}
\label{sec:furtherresultsinf}

Figure~\ref{fig:Accuracy_inf_inf} shows results against $\ell_\infty$-norm constrained PGD attacks when networks explicitly either use $\ell_\infty$-norm constrained adversarial training or, equivalently, regularize the ($\infty$, 2)-operator norm of the network.
The conclusions remain the same for all the experiments we conducted.
\medskip

\begin{figure}[h!]
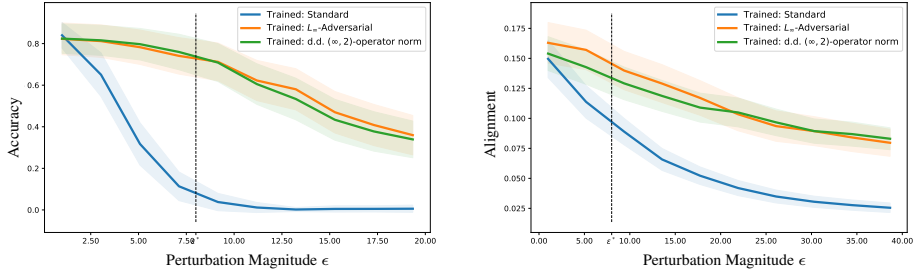

\centering
\begin{tabular}{@{}l@{\hspace{1pt}}c@{\hspace{15pt}}l@{\hspace{1pt}}c@{}}
\begin{turn}{90}\hspace{30pt} {\tiny Accuracy} \end{turn} & \adjincludegraphics[width=0.4\linewidth, trim={{0.037\width} {0.085\height} {0.0\width} {0.0\height}},clip]{plots/accuracy_inf_inf} &
\begin{turn}{90}\hspace{30pt} {\tiny Alignment} \end{turn} & \adjincludegraphics[width=0.4\linewidth, trim={{0.037\width} {0.085\height} {0.0\width} {0.0\height}},clip]{plots/alignment_inf_inf} \\[-1mm]
& \quad{\tiny Perturbation Magnitude $\epsilon$} & & \quad{\tiny Perturbation Magnitude $\epsilon$}
\end{tabular}
\vspace{-1mm}
\caption{
(Left) Classification accuracy for an $\ell_\infty$-norm trained network on $\ell_\infty$-norm perturbations with $\epsilon$ (measured in 8-bit). 
(Right) Alignment of $\ell_\infty$-PGA adversarial perturbations with dominant singular vector of $\Jac{\x}$ as a function of perturbation magnitude $\epsilon$.
The dashed vertical line indicates the $\epsilon$ used during adversarial training. Curves were aggregated over $2000$ samples from the test set.}
\label{fig:Accuracy_inf_inf}
\end{figure}

\subsection{SVHN}
\label{sec:svhn}

Figure~\ref{fig:Accuracy_svhn} shows results against $\ell_2$-norm constrained PGD attacks on the SVHN dataset.
As can be seen, the behavior is very comparable to our analogous experiment in the main section.
\medskip

\begin{figure}[h!]
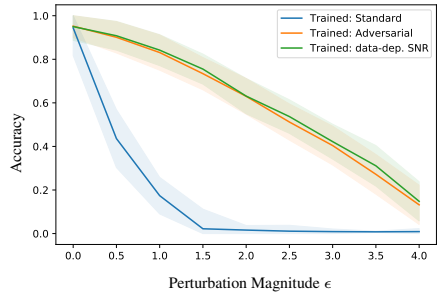

\centering
\begin{tabular}{@{}l@{\hspace{1pt}}c@{\hspace{15pt}}l@{\hspace{1pt}}c@{}}
\begin{turn}{90}\hspace{30pt} {\tiny Accuracy} \end{turn} & \adjincludegraphics[width=0.4\linewidth, trim={{0.045\width} {0.095\height} {0.0\width} {0.0\height}},clip]{plots/accuracy_svhn} \\
& \quad{\tiny Perturbation Magnitude $\epsilon$}
\end{tabular}
\vspace{-1mm}
\caption{
Classification accuracy on SVHN. Curves were aggregated over $2000$ samples from the test set.}
\label{fig:Accuracy_svhn}
\end{figure}

\end{document}